\documentclass[twoside]{article}
\usepackage[preprint]{aistats2026}
\usepackage[T1]{fontenc}
\usepackage{graphicx}
\usepackage{amsmath,amsfonts, amsthm}

\usepackage{nicefrac}
\usepackage{url}

\usepackage{bm}
\allowdisplaybreaks
\usepackage{tabularx}
\newcommand\cone{\mathcal{N}_{\X}}
\newcolumntype{Y}{>{\centering\arraybackslash}X}

\usepackage{nicefrac}

\newcommand{\sumtsw}{\sum_{t=1}^{T-1}\|\vec x_{t} - \vec x_{t-1}\|}

\newcommand{\sumT}{\sum_{t=1}^T} 
\newcommand{\sumTO}{\sum_{t=1}^{T-1}}

\newcommand{\dt}{\|\vec{u}_{t+1} - \vec{u}_t\|}

\newcommand{\T}{\mathcal{T}}
\newcommand{\X}{\mathcal{X}}
\newcommand{\R}{\mathcal{R}}

\newcommand{\dtp}[2]{\langle {#1}, {#2} \rangle}

\renewcommand{\vec}[1]{\bm{#1}}

\DeclareMathOperator*{\argmin}{argmin}

\usepackage{algorithm}
  \makeatletter
  \renewcommand{\ALG@name}{Alg.}
  \makeatother
\usepackage{algorithmic}

\usepackage{comment}
\usepackage{soul}
\usepackage{mathtools}
\mathtoolsset{showonlyrefs}
\newtheorem{theorem}{Theorem}
\newtheorem{proposition}[theorem]{Proposition}
\newtheorem{lemma}[theorem]{Lemma}

\usepackage{wrapfig}
\usepackage{multicol}
\usepackage{enumitem}
\usepackage[round]{natbib}

\usepackage[hidelinks]{hyperref}
\begin{document}

\twocolumn[
\runningtitle{Partially Lazy Gradient Descent}
\aistatstitle{Partially Lazy Gradient Descent for Smoothed Online Learning}
\runningauthor{Mhaisen, Iosifidis}
\aistatsauthor{ Naram Mhaisen$^{1}$ \And George Iosifidis$^{1}$}
\aistatsaddress{\vspace{2mm}$^{1}$Faculty of Electrical Engineering, Mathematics and Computer
Science.\\ TU Delft, Netherlands. } ]
\begin{abstract}
We introduce \textsc{$k$-lazyGD}, an online learning algorithm that bridges the gap between greedy Online Gradient Descent (OGD, for $k=1$) and lazy GD/dual-averaging (for $k=T$), creating a spectrum between reactive and stable updates. We analyze this spectrum in Smoothed Online Convex Optimization (SOCO), where the learner incurs both hitting and movement costs. Our main contribution is establishing that laziness is possible without sacrificing hitting performance: we prove that \textsc{$k$-lazyGD} achieves the optimal dynamic regret $\mathcal{O}(\sqrt{(P_T+1)T})$ for any laziness slack $k$ up to $\Theta(\sqrt{T/P_T})$, where $P_T$ is the comparator path length. This result formally connects the allowable laziness to the comparator's shifts, showing that \textsc{$k$-lazyGD} can retain the inherently small movements of lazy methods without compromising tracking ability. We base our analysis on the Follow the Regularized Leader (FTRL) framework, and derive a matching lower bound. Since the slack depends on $P_T$, an ensemble of learners with various slacks is used, yielding a method that is provably stable when it can be, and agile when it must be.
\end{abstract}

\section{INTRODUCTION}

We study the problem of \emph{Smoothed Online Convex Optimization} (SOCO)~\citep{cesa2013online, chen2018smoothed, zhao2020understand}, where a learner interacts with a possibly adversarial environment over $T$ rounds. At each round $t = 1, \dots, T$, the learner selects an action from a convex compact set $\vec{x}_t \in \mathcal{X}$, then incurs a \emph{hitting cost} $f_t(\vec{x}_t)$ and a \emph{movement cost} $m(\vec{x}_t, \vec{x}_{t-1})$ penalizing changes in successive decisions. We follow the OCO convention in which the hitting cost function $f_t(\cdot)$ is revealed only \emph{after} $\vec{x}_t$ is chosen. As is standard in SOCO, we set the movement cost to be the $\ell_2$ distance,
$m(\vec{x}_t, \vec{x}_{t-1}) = \|\vec{x}_t - \vec{x}_{t-1}\|$, though our analysis naturally extends to time-varying convex movement costs that are Lipschitz-continuous.

Our metric is \emph{dynamic regret with switching cost}:
\begin{align}
    \label{eq:regret}
    \hspace{-5mm}\mathcal{R}_T 
    \doteq 
    \sum_{t=1}^T \big(f_t(\vec{x}_t)-f_t(\vec{u}_t)\big)\;+\!\sumTO\|\vec{x}_{t+1} - \vec{x}_t\|.
\end{align}
The first sum is the \emph{dynamic regret} against an arbitrary comparator sequence $\{\vec{u}_t\}_{t=1}^T$, and the second one is the learner's \emph{switching cost}. The switching cost of the comparator is termed ``path length'' $
P_T \doteq \sumTO \|\vec{u}_{t+1} - \vec{u}_t\|.$\footnote{Some formulations subtract $P_T$ from $\R_T$, but this will not change the final order. See, e.g.,  \citet{zhang2021revisiting}.} 

The $\R_T$ metric jointly measures the learner’s hitting performance, that is, how well $\vec{x}_t$ competes with $\vec{u}_t$ in terms of $f_t(\cdot)$, and the stability of its actions over time. The goal is to design an algorithm \textsc{ALG} achieving
$\mathcal{R}_T^{\textsc{ALG}} = \mathcal{O}(\sqrt{(P_T+1)T})$, where $\mathcal{R}_T^{\textsc{ALG}}$ refers to the regret defined in \eqref{eq:regret} when the actions $\vec x_t$ are generated via \textsc{ALG}. 
Such a rate is minimax-optimal: it matches the best possible bound known for dynamic regret \citep[Thm. 2]{zhang2018adaptive},  which applies to the more difficult metrics that include switching costs.

\paragraph{Two variants.} Among the foundational methods in OCO, OGD~\citep{zinkevich2003online} plays a central role.
It admits two classical update rules: a greedy variant and a lazy variant.
Let $\vec{g}_t \in \partial f_t(\vec{x}_t)$ be a subgradient of $f_t(\cdot)$ at $\vec{x}_t$, and let $\sigma > 0$ denote a \emph{regularization} parameter. The greedy variant (often simply called gradient descent) takes a step from the previous iterate in the direction of $-\vec{g}_t$, scaled by $1/\sigma$,\footnote{Most presentations of~\ref{eq:gd} use a learning rate parameter $\eta$.
Here, we instead write $\sigma = 1/\eta$, which is more natural for our later derivations in the FTRL framework.} and then projects back onto the feasible set:
\begin{equation}
\vec{x}^{\text{G}}_{t+1}
= \Pi\left( \vec{x}^{\text{G}}_t - \frac{1}{\sigma} \vec{g}_t \right),
\label{eq:gd}
\tag{\textsc{GD}}
\end{equation}
where $\Pi(\cdot)$ denotes the Euclidean projection onto the convex set $\mathcal{X}$:
$\Pi(\vec{z}) \doteq \arg\min_{\vec{x} \in \mathcal{X}} \| \vec{x} - \vec{z} \|.$

In contrast, the lazy variant, known as dual averaging~\citep{xiao2009dual}, first accumulates \emph{all} past subgradients and then projects their scaled direction, making the update independent of the previous iterate:
\begin{align}
\vec{x}^{\text{L}}_{t+1}
= \Pi\left( - \frac{\vec{g}_{1:t}}{\sigma} \right),
\label{eq:lazy-gd}
\tag{\textsc{LazyGD}}
\end{align}
where $\vec g_{a:b}$ denotes the sum $\sum_{\tau=a}^b \vec g_\tau$. The two variants may coincide in both their update rules and regret guarantees.
In the unconstrained setting, $\mathcal{X} = \mathbb{R}^d$, the projection operator reduces to the identity, and the greedy update \eqref{eq:gd} can be simply unrolled into the lazy one \eqref{eq:lazy-gd}.\footnote{If $\sigma$ is not constant, this equivalence does not hold even in unconstrained settings. See \citet[Sec. 1]{liu2025dual}.}
In constrained domains, the iterates generally differ, yet both variants achieve the same optimal \emph{static} regret (i.e., when $\vec{u}_t = \vec{u}^\star, \forall t$)  of $\mathcal{O}(\sqrt{T})$. They also share the same \emph{worst-case} switching bound:
$\|\vec{x}_t - \vec{x}_{t+1}\| \leq G/\sigma$, where $G$ is the Lipschitz constant, $\|\vec g_t\| \leq G, \forall t$. This bound is a direct consequence of the non-expansiveness of the projection operator (see Appendix~\ref{app:non-expan}). Yet, empirical evidence shows that, in many cases, this switching cost bound is considerably looser for \ref{eq:lazy-gd} than it is for \ref{eq:gd}.
\paragraph{Switching behavior.} The lower switching cost of \ref{eq:lazy-gd} is arguably intuitive; accumulating gradients over time is inherently more stable than reacting to each new one individually. This distinction becomes evident in the illustrative examples of Fig.~\ref{fig:example}. In example~$(i)$, driven by an initial set of fixed gradients, both variants incur the same switching cost required to reach the boundary of the $\ell_1$-ball, paying a switching cost of $1$. However, once the gradients begin rotating thereafter, the greedy update of~\ref{eq:gd} continues to chase the most recent direction, leading to a continually increasing, though sublinear, movement cost.
In contrast, such rotations do not alter the optimum when appended to the aggregated gradients, causing the lazy iterates to stall. As a result, the lazy variant pays \emph{zero} switching cost after reaching the boundary. Moreover, since gradients are orthogonal, chasing the most recent one yields no advantage in hitting costs.

\begin{figure}
    \centering
    \includegraphics[width=0.98\linewidth]
    {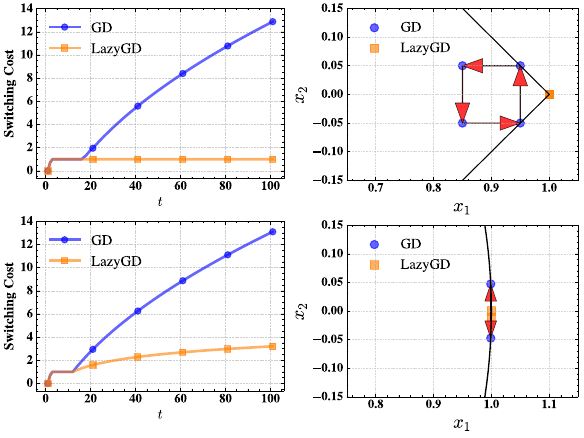}
    \label{fig:example}
    \vspace{-5mm}
\caption{Switching in example~($i$,\! top), showing \emph{staleness}, and example~($ii$,\! bottom), showing \emph{stability}. Left: switching cost. Right: Snapshots over 4 (top) and 2 (bottom) rounds: greedy updates move continuously, whereas lazy updates remain still or move minimally.}
    \vspace{-4mm}
\end{figure}
Example~$(ii)$ highlights another advantage. Here, the feasible set is the $\ell_2$-ball, and every new gradient forces both variants to move. Gradients have a fixed horizontal component and an alternating vertical one. The greedy update responds fully to each incoming gradient, as dictated by its update rule. The lazy update, however, attenuates movement in proportion to the magnitude of the \emph{aggregated} gradients, effectively leveraging their correlation. Under sufficient gradient correlation, as in this example, the movement of the lazy variant is greatly reduced. As in Example~$(i)$, chasing the most recent gradient yields no improvement in hitting cost since only the fixed horizontal component contributes to loss reduction. Additional details on both examples are provided in Appendix~\ref{app:more-details}.

These two examples illustrate each a key property underlying \textsc{LazyGD}’s superior switching behavior:
$(i)$ \emph{iterate staleness}: updates halt as long as the incoming gradient does not alter the minimizer of their accumulated sum; and
$(ii)$ \emph{iterate stability}: even when the minimizer is altered, the movement's magnitude is attenuated proportionally to the size of the accumulated sum. And while in the worst case  \ref{eq:gd} and ~\ref{eq:lazy-gd} share the same movement bound of $G/\sigma$, the difference lies in sensitivity: in the former, each new gradient can trigger a movement of such size \emph{regardless} of its effect on the \emph{direction} of accumulated gradients, or on the \emph{size} thereof, whereas \textsc{LazyGD} suppresses movement whenever either of the above properties is in effect. 

\begin{figure*}[t]
    \centering
    \includegraphics[width=0.825\linewidth]{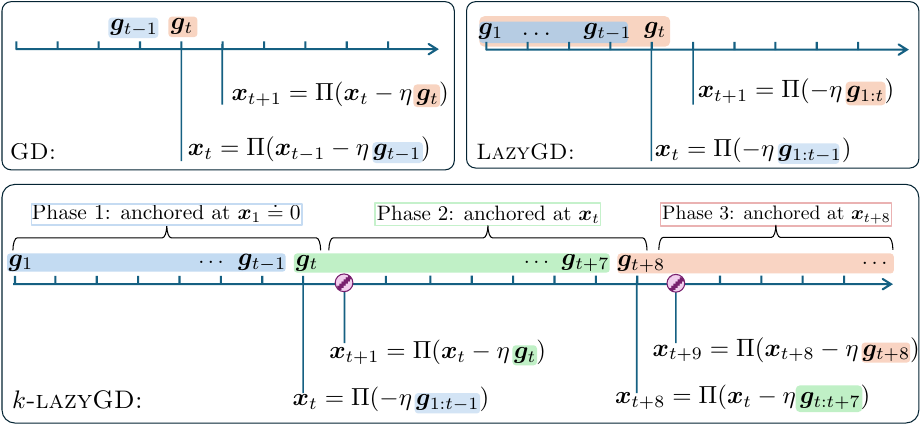}
\caption{
Top-left: \textsc{GD} updates from the previous iterate using only the most recent gradient. 
Top-right: \textsc{LazyGD} updates from the origin (assumed $0$) using cumulative gradients. 
Bottom: \textsc{$k$-lazyGD} combines these two views: it partitions time into phases (of length $k=8$ here), anchors each phase at its starting iterate, and updates from the anchor using the cumulative gradients within that phase. We use $\eta = 1/\sigma$ for clarity.
}
    \label{fig:fig}
\end{figure*}

\paragraph{Limits of lazy updates.} Despite their stable behavior and minimal movement, lazy methods have not been used in SOCO, and for a good reason. An impossibility result by~\citet{jacobsen2022parameter} shows that \emph{any} method of the~\ref{eq:lazy-gd} form suffers \emph{linear} dynamic regret, hence $\R_T=\Omega(T)$. This lower bound holds for \emph{any} non-zero path length. That is, even when the environment changes minimally, lazy methods \emph{fail} dramatically in the hitting cost. As recently observed in~\citep{pruning-icml}, the root cause is that accumulating past gradients can mask subtle but important changes in the gradient direction, rendering the lazy iterate effectively unresponsive to shifts that are critical for minimizing dynamic regret. 

This barrier highlights a tension between stability and responsiveness. Minimizing switching favors aggregating gradients over time, whereas minimizing dynamic regret requires continual adaptation to each new gradient.  
Neither extreme is ideal: greedy updates are overly reactive, while fully lazy ones can fail to track even simple non-stationary benchmarks. This paper therefore asks: \emph{what is the maximum laziness compatible with optimal dynamic regret?} That is, how far can a learner accumulate past gradients while still tracking non-stationary comparators.

\paragraph{Unifying view.} Motivated by this question, we explore the intermediate space of \emph{partial} laziness, interpolating between greedy and lazy updates.  
One natural approach is to study an update of the form: 
\begin{equation}
\begin{aligned}
    & \vec x_{t+1} = \Pi\left(\vec x_{t-n_t} - \frac{\vec g_{t-n_t:t}}{\sigma}\right) \\
    & n_t = (t-1) \!\!\!\mod (k),
\end{aligned}
\tag{\textsc{$k$-lazyGD}}
\label{k-lazy}
\end{equation}
for a given \emph{laziness slack} $k\in \{1, \dots, T\}$. Recall that $\vec g_{t-n_t:t}$ is the sum of $\vec g_\tau$ from $\tau=t-n_t$ to $\tau=t$.
This scheme recovers the greedy update of ~\ref{eq:gd} when $k = 1$ (since in this case $n_t=0, \forall t$), and the lazy update of ~\ref{eq:lazy-gd} when $k = T$,\footnote{Assuming the null initialization $\vec x_1=0$.} (since in this case $n_t=t-1, \forall t$). For intermediate values of $k$, the rule \ref{k-lazy} partitions the horizon into $T/k$ phases of length $k$ each\footnote{Only the last phase is of length $T \mod k$.}. Within each phase, gradients are accumulated as in~\ref{eq:lazy-gd}, but, crucially, the accumulation is reinitialized at the (dual mapping of) the last iterate from the preceding phase. Fig.\ref{fig:fig} visualizes an example of how the iterates are updated. To our knowledge, this phased form, and the advantages thereof, were not investigated in the literature. 

\paragraph{Contributions.} Our first result is showing that the non-standard update of~\ref{k-lazy} can be cast as an instance of the FTRL framework using the ``pruning'' trick from~\cite{pruning-icml}, who showed that responsiveness is achieved via a principled way of pruning past gradients at a given $t$. 
Here, we generalize this mechanism: we accumulate as many gradients as possible, hence preserving the staleness and stability properties, and \emph{only} trigger pruning when further accumulation would otherwise hinder the algorithm's agility and affect the dynamic regret. This FTRL reduction also enables us to reuse existing analyses in a modular way, yielding a cleaner proof structure.
  
We then turn to the question of how ``much'' laziness (stability) can be tolerated without compromising the dynamic regret. To that end, we derive a universal lower bound for \ref{k-lazy} methods that extends the impossibility result of~\citet[Thm.~1]{jacobsen2022parameter}, which corresponds to the extreme case $k = T$.
Our result reveals that the laziness slack may be as large as $k^\star=\Theta(\sqrt{T / P_T})$ while still remaining responsive enough to maintain optimal dynamic regret. Beyond this threshold, the learner becomes insufficiently responsive to track non-stationary comparators, an insight that is new to the OCO literature.

Next, we quantify the benefit of allowing slacks larger than the current $k{=}1$ standard in SOCO. We formalize the two properties of lazy updates: \emph{staleness}, which prevents unnecessary movement, and \emph{stability}, which attenuates unavoidable movement, both of which benefit from a larger $k$. We then show that \ref{k-lazy} still achieves optimal dynamic regret, while having these two properties. Our analysis extends recent work on the FTRL framework with dynamic comparators, and reveals that optimal regret can be attained without \emph{fully} discarding the structural benefits of laziness.

Lastly, the optimal choice of laziness slack depends on the comparator's non-stationarity, quantified in $P_T$. While one could specify an upper bound on $P_T$ a priori, achieving simultaneous optimality across all comparator sequences necessitates adaptivity. This challenge is addressed by the ensemble (meta-learning) framework of~\citet{zhao2024adaptivity}, 
where multiple instances of~\ref{eq:gd} with different \emph{learning rates} are run in parallel and aggregated. The dependence of the \emph{laziness slack} on $P_T$ is analogous. Accordingly, we adopt the meta-learning principle, but 
across an ensemble of~\ref{k-lazy} instances, each with a distinct $(k,\sigma)$ pair.

\section{\ref{k-lazy} as an FTRL INSTANCE}
We introduce \ref{k-lazy} as an instance of the general FTRL framework. This is achieved via a flexible pruning rule, which enables us to control the amount of history (past gradients) retained. Formally, the FTRL update, executed on the linearized extended real value function $\bar f_t(\cdot) = f_t(\cdot) + I_\X(\cdot)$, with an $\ell_2$ regularizer is: 
\begin{align}
    \label{eq:k-lazy-update-FPRL}
    \hspace{-7mm}{\vec{\hat x}}_{t+1} \!= \Pi(\vec y_{t+1}),\ \ \vec{y}_{t+1}\!= \argmin_{\vec{x}\in\mathbb{R}^{d}} \dtp{\vec p_{1:t}}{\vec x} \!+\! \frac{\sigma}{2}\|\vec x\|^2,
\end{align}
where $\vec{p}_{1:t}$ is referred to as the state vector and is calculated as the aggregation of the subgradients of $\bar f_t(\vec x_t)$:
\begin{align}
&\vec{p}_t = \vec g_t + \vec g_t^{I},\\
&\vec{g}_t \in \partial f_t(\vec {\hat x}_t),\quad \vec{g}^{I}_t \in \partial I_\mathcal{X}(\vec{\hat x}_{t}) = \cone(\vec{\hat x}_{t}).\\[-4ex]
\end{align}
\vspace{0.5mm}$\cone(\vec{x})$ denotes the normal cone of the set $\X$ at $\vec x$.\footnote{The \emph{normal cone} at $\vec{x} \in \mathcal{X}$ is the set of vectors that form a non-acute angle with every feasible direction from $\vec{x}$: $\cone(\vec{x}) \doteq \{\vec{g} \big| \dtp{\vec g}{\vec y - \vec x} \leq 0, \forall \vec y \in \mathcal{X}\}.$}
The solve-then-project routine in \eqref{eq:k-lazy-update-FPRL} can also be equivalently written as a direct minimization over $\X$ \citep[Thm. 11]{mcmahan-survey17}:
\begin{align}
    \vec{\hat x}_{t+1}\!= \argmin_{\vec{x}\in\X} \dtp{\vec p_{1:t}}{\vec x} \!+\! \frac{\sigma}{2}\|\vec x\|^2. \label{eq:k-lazy-update-FPRL2}
\end{align}
Note that the standard choice is $\vec g_t^I=0, \forall t$, which leads to the well-known FTRL iteration.
However, we observe that this is not the only choice. Specifically, we propose the following: 
for $t=1$, set $\vec{g}_1^I = 0$. $\forall t \geq 2$:
\begin{align}
\label{eq:pruning-cond-n}
 \!\!\!\!\vec g_t^{I} = 
    \begin{cases}
       -\vec p_{1:t-1} - \sigma \vec {\hat{x}}_{t} & \text{if }(\vec y_{t} \notin \mathcal{X}  \ \land\  n_t=0), 
       \\
        0 & \text{otherwise,}
    \end{cases}
\end{align}
where $n_t$ is a counter used to control the \emph{rate} of pruning to be at most every $k$ steps: $n_t = (t-1) \!\!\!\mod (k)$. 

Note that the choices made in \eqref{eq:pruning-cond-n} for  $\vec g_t^I$ are always valid; when $\vec y_t \notin \mathcal{X}$, the projection lands on the boundary: $\vec x_t \in \textbf{bd}(\mathcal{X})$, since $\mathcal{X}$ is compact. The normal cone $\cone(\vec {\hat x}_t)$ then contains nonzero vectors. In particular, it includes $-(\vec p_{1:t-1}  + \sigma \vec {\hat x}_t)$. This follows directly from the optimality condition for the constrained update in \eqref{eq:k-lazy-update-FPRL2} (see \citet[Thm. 3.67]{beck-book}). In all cases, $0$ is always a valid subgradient since all cones must contain the zero vector.

The choice of $\vec g_t^I$ makes explicit what we mean by \emph{pruning}: discarding the accumulated state vector $\vec p_{1:t-1}$ and replacing it with a representative one $\sigma \vec {\hat x}_t$.\footnote{This is the dual mapping of $\vec x_t$, with  map $\nabla \tfrac{1}{2}\|\vec x\|^2$.}

\begin{theorem}
    \label{thm:equivalence}
    The iterates of \ref{k-lazy}, $\{\vec x_t\}_{t=1}^T$, coincide with those of the FTRL routine defined above in \eqref{eq:k-lazy-update-FPRL} and \eqref{eq:pruning-cond-n}, $\{\vec {\hat x}_t\}_{t=1}^T$ . Namely:
    \begin{align}
        \vec{\hat{x}}_{t+1} = \Pi\big(\vec x_{t-n_t} - \frac{\vec g_{t-n_t:t}}{\sigma}\big), \quad  n_t = (t-1) \!\!\!\mod (k).
    \end{align} 
\end{theorem}
\vspace{-6mm}
\begin{proof}
We proceed by induction. Let the initialization be equal $\vec{\hat x}_1 = \vec{x}_1 = 0 $, and let the Induction Hypothesis (IH) be $\vec{\hat x}_t = \vec{x}_t$, for some $t$. We now prove the inductive step: $\vec{\hat x}_{t+1} = \vec{ x}_{t+1}$.

Note that both, \ref{k-lazy} and FTRL in \eqref{eq:k-lazy-update-FPRL}, perform a Euclidean projection of a given vector. Since the projection is unique, it suffices to show the equivalence of those vectors. That is, we want to show that:
\begin{align}
    \label{eq:unconstrained-vectors}
    \vec y_{t+1} = \vec x_{t-n_t} - \frac{\vec g_{t-n_t:t}}{\sigma}.
\end{align}
By the optimality conditions of the unconstrained minimization in  \eqref{eq:k-lazy-update-FPRL} we know that, for all $t$, 
\begin{align}
    \label{eq:FPRL-sol-unconst}
    \vec y_{t+1} = - \frac{\vec p_{1:t}}{\sigma}.
\end{align}
Hence, by substituting \eqref{eq:FPRL-sol-unconst} in \eqref{eq:unconstrained-vectors}, it suffices to show: 
\begin{align}
    \label{eq:to-show}
    - \frac{\vec p_{1:t}}{\sigma} \;=\; \vec x_{t-n_t} - \frac{\vec g_{t-n_t:t}}{\sigma}.
\end{align}
We distinguish two cases based on $n_t$.

\textbf{Case 1:} $n_t = 0$.
\textit{subcase 1.a:} $\vec g_t^I = 0$. From the condition in \eqref{eq:pruning-cond-n}, this implies $\vec y_t \in \X$, and hence the projection $\vec {\hat x}_t$ is $\vec y_t$. It follows that
\begin{align}
    \vec {\hat x}_t \stackrel{\eqref{eq:FPRL-sol-unconst}}{=} \frac{\vec p_{1:t-1}}{\sigma} \stackrel{\text{IH}}{=} \vec x_t.
\end{align}
\eqref{eq:to-show} follows by adding $-\frac{\vec{g}_t}{\sigma}$ to both sides and recalling that in this subcase, we have $ n_t = 0$ and $\vec g_t^I = 0$.

\textit{subcase 1.b:} $\vec g_t^I = -\vec p_{1:t-1} - \sigma \vec {\hat x}_t$. 
In this case, we write
\begin{align}
    \hspace{-2mm}\frac{-\vec{p}_{1:t}}{\sigma}= \frac{-\vec{p}_{1:t-1}}{\sigma} -\frac{\vec g_t}{\sigma} -\frac{\vec{g}_t^I}{\sigma} = \vec{\hat x}_{t} - \frac{\vec{g}_t}{\sigma} \stackrel{\text{IH}}{=} \vec x_{t} - \frac{\vec{g}_t}{\sigma},
\end{align}
showing that \eqref{eq:to-show} follows from the subcase conditions.
From both subcases, we have that for any index $\tau$, 
\begin{align}
    \label{eq:0-means}
    n_\tau = 0\;\Rightarrow\;
    \vec x_\tau - \frac{\vec{g}_\tau}{\sigma} = \frac{-\vec p_{1:\tau}}{\sigma}.
    \\[-4ex]
\end{align}
\textbf{Case 2:} $0<n_t\leq k-1$
In this case, let $\tau \doteq t - n_t$ denote the last round in which pruning is attempted. By construction, we have $n_\tau = 0$, so Case 1 applies at round $\tau$. Hence, from \eqref{eq:0-means}, we obtain:
\begin{align}
    \label{eq:to-show-case2}
    - \frac{\vec p_{1:t-n_t}}{\sigma} \;=\; \vec x_{t-n_t} - \frac{\vec g_{t-n_t}}{\sigma}.
\end{align}
Now observe that by the pruning rule \eqref{eq:pruning-cond-n}, no pruning occurs in the $k-1$ steps following $\tau$. That is, pruning is inactive on the interval $\{\tau+1, \ldots, \tau+(k-1)\} $. Substituting $\tau = t - n_t$, we find that pruning is inactive during the interval $\{t - n_t + 1, \ldots, t\}$.
Thus
\begin{align}
    \vec{g}_{t-n_t+1}^I = \cdots = \vec{g}_t^I = 0.
\end{align}
Using this fact, and adding $-{\vec g_{t-n_t+1:t}}/{\sigma}$ to both sides of \eqref{eq:to-show-case2}, we arrive at  \eqref{eq:to-show}.
\end{proof}
In summary, \ref{k-lazy} is an FTRL instance with a specific choice of linearization $\vec p_t$. This perspective aids both the analysis and further generalization.

\section{A LOWER BOUND FOR LAZY ALGORITHMS}
Since the presented FTRL variant is not covered by existing lower bounds\footnote{Except the universal $\sqrt{(P_T{+1})T}$ (independent of $k$).},
we construct an adversarial instance showing that any~\ref{k-lazy} algorithm must incur dynamic regret linear in $k$. Building on~\citep[Thm.~2]{jacobsen2022parameter}, 
we generalize their argument for any laziness slack $k$ and path length $P_T$.
\begin{theorem}\label{thm:lb}
For any \ref{k-lazy} learner and any $\tau \in [1, 2R(T-1)]$, where $R$ is the radius of $\X$, there exists  a  comparator $\{\vec u_t\}_{t=1}^T$, with $P_T=\lfloor \tau \rfloor$, and a sequence $\{f_t(\cdot)\}_{t=1}^T$ such that $\mathcal{R}_T \;=\; \Omega\!\left(k\,P_T\right)$.
\end{theorem}
\begin{proof}
    Let the decision space be the line segment $\mathcal{X} = [-1, 1]$.
Let the set of sub-intervals be defined as
\[
    \mathcal{T}_i = [(i-1)k+1, \dots, ik] \quad \text{for } i = 1, \dots, \lfloor T/k\rfloor .
\]
For every interval $\mathcal{T}_i$, the adversary defines the cost as $f_t(x) = g_t\:x$, where  $g_t$ is a simple square wave:
\begin{align}
\label{eq:cost-const}
\!\!\!\!g_t = \begin{cases}
+1, & \text{for } t \in \left[(i-1)k+1,\; (i-1)k + \frac{k}{2}\right], \\
-1, & \text{for } t \in \left[(i-1)k + \frac{k}{2}+1,\; ik\right],
\end{cases}
\end{align}
where $k$ is assumed even w.l.o.g.\footnote{For an odd $k$, we construct the square wave using the even $k-1$, and set the $k^{\text{th}}$ cost to $0$. We defer this detail to Appendix~\ref{app:LB-details} since it is mainly index manipulation.}. Concatenating all intervals $\{\mathcal{T}_i\}_i$ yields the full horizon.

By construction, the cost sequence is $k$-periodic ($g_t=g_{t+k}$) with zero sum per base period. Hence, any window of length $ik$ sums to zero:
\begin{align}
    \label{eq:sum-base-period}
    \hspace{-8mm}\sum_{t=j+1}^{j+ik}\! g_t = 0,\,  
    i{\in}\{1,\dots,\lfloor T/k\rfloor\},\, j{\in}\{0,\dots,T{-}ik\}.
\end{align}
Moreover, since the sequence begins with the positive half, sums starting at $1$ are nonnegative: $g_{1:t}\ge0,~\forall t$.

We claim that, with the cost sequence above, and for any~\eqref{k-lazy} variant, the iterates satisfy $x_t \le 0$ for all $t$. 
It is enough to show that $g^I_t = 0$ for every $t$, because in such case the update becomes
\begin{align}
    &\!\!\!x_t = \Pi(y_t) \stackrel{(a)}{=} \min\big(1, \max(-1, y_t)\big)\stackrel{\eqref{eq:k-lazy-update-FPRL}}{=}
    \min\big(1,
    \\
    \,
    &\max(-1, \frac{-g_{1:t-1}}{\sigma})\big)
    \stackrel{(b)}{=}\max\big(-1, \frac{-g_{1:t-1}}{\sigma}\big) \leq 0.
\end{align}
Here, $(a)$ uses that projecting onto $[-1,1]$ is a sign-preserving clipping, and $(b)$ follows from $g_{1:t}\geq0$.

\begin{figure}[t]
    \!\!\!\!\begin{minipage}{0.67\linewidth}
        \includegraphics[width=\linewidth]{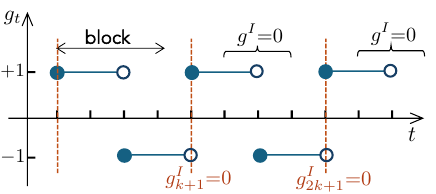}
    \end{minipage}%
    \hfill
    \begin{minipage}{0.33\linewidth}
    \vspace{-4mm}\caption{The constructed sequence ($k=4$). $g_t^I$ is zero within blocks by \eqref{eq:pruning-cond-n}, and at block starts since $y=0$ then.}    \label{fig:LB}
    \end{minipage}
    \vspace{-2mm}
\end{figure}

It remains to show that $g_t^I=0$ for all $t$. 
Indeed, the construction in~\eqref{eq:pruning-cond-n} implies that $g_t^I$ can only be nonzero at indices of the form $t=jk{+}1 \le T$ for some $j \ge 0$:
\begin{align}
\label{eq:lb-rest-is-0}
\hspace{-8mm}g^I_t = 0 \quad \forall\, t \notin \mathcal{I} \doteq \{\,jk{+}1 : j \ge 0,\, jk{+}1 \le T\,\}.
\end{align}
The proof therefore reduces to showing $g^I_t=0$ for all $t \in \mathcal{I}$. These indices mark the start of each block (Fig.~\ref{fig:LB}). We use induction over blocks.

\textit{base case:} $g_1^I=0$ is true by definition. 
\textit{Induction Hypothesis (IH)}: $g^I_{lk+1} = 0, \ l \in\{1,\dots,j\}$. 
Now we show the inductive step. Intuitively, we show that at each pruning opportunity  (the dashed lines in Fig.~\ref{fig:LB}), the unconstrained iterate in \eqref{eq:k-lazy-update-FPRL} has already returned to the feasible set, so no pruning occurs: $y_{(j+1)k+1}$
\vspace{1mm}
\begin{align}
    \!\!&{=}\frac{-p_{1:(j+1)k}}{\sigma} 
    \stackrel{\eqref{eq:lb-rest-is-0}}{=}\frac{1}{\sigma}\big(
        -\underbrace{p_{1:jk+1}}_{\mathclap{\substack{\text{inside \& start}\\\text{of prev. blocks}}}}
        \,-\,\overbrace{g_{jk+2:(j+1)k}}^{\smash{\text{inside last block}}}\!\!\big)
    \\&\stackrel{\text{(IH$\And$\eqref{eq:lb-rest-is-0})}}{\;=\;}\frac{-g_{1:(j+1)k}}{\sigma} 
    \stackrel{\eqref{eq:sum-base-period}}{\;=\;} 0. 
\end{align}
This shows that $y_{(j+1)k+1} \in \X$, implying that $g^I_{(j+1)k+1} = 0 $, and completing the inductive step.

We proceed to calculate the regret of these non-positive actions.
Let $\tau' \doteq \lfloor \tau \rfloor$ denote our comparator’s switch budget.
We declare $\T_1$ \emph{active}, and then mark a block active if we can spend \emph{two} switch units inside it. Thus, the number of active blocks $A$ is
\begin{equation}
\label{eq:num-active}
A = 1 + \left\lfloor \frac{\tau'-1}{2} \right\rfloor
= \left\lceil \frac{\tau'}{2} \right\rceil
\end{equation}
Fix an active block $\T_i$.
By the argument above, $x_t\le 0$.
\begin{align}
\sum_{t\in\T_i} g_t x_t
&\ge \sum_{t=(i-1)k+1}^{(i-1)k+k/2} g_t x_t
\;\ge\; -\sum_{t=1}^{k/2} 1
\;=\; -\frac{k}{2}.
\label{eq:learner-half}
\end{align}
where the first equality is because $g_tx_t$ is positive in the second half of $\T_i$.

For the comparator cost, in $\T_1$, the comparator starts at $u_t=-1$ and switches to $+1$ at the midpoint, hence incurring loss $-1$ at every step:
$\sum_{t\in\T_1} g_t u_t = -k.$
In every subsequent active block $\T_i$ ($i\ge 2$), the comparator spends \emph{two} switches: reset $u_t=-1$ at the block’s first step and flip to $+1$ at its midpoint. This again yields
$\sum_{t\in\T_i} g_t u_t = -k.$
Summing over active blocks
\[
\R_T
\ge
\frac{k}{2}\,A
=
\frac{k}{2}\,\left\lceil \frac{\tau'}{2} \right\rceil
\ge
\frac{k}{4}\,\tau',
\]
from which the lower bound follows.
\end{proof}
\textbf{Remarks.} The result clarifies why any $P_T \geq 1$ forces linear regret for \ref{eq:lazy-gd} ($k{=}T$). It also reveals a gap in admissible laziness: while \ref{eq:gd} ($k{=}1$) achieves $\sqrt{(P_T{+}1)T}$ rate, the lower bound \emph{suggests} that this rate may still be attainable with a larger slack of up to $\Theta(\sqrt{T/P_T})$. We establish next a matching upper bound, showing that this threshold indeed marks the true tradeoff between laziness and regret. But first, we formalize how a larger $k$ improves switching cost. 


\section{ANALYSIS OF \ref{k-lazy}}
\textbf{Assumptions.} Throughout, we make the standard Lipschitzness and compact-domain assumptions: \(\mathcal{X}\subset\mathbb{R}^d\) is convex and \(\mathcal{X}\subseteq\{\vec x:\|\vec x\|\le R\}\); (sub)gradients satisfy \(\|\vec g_t\|\le G\) for all \(t\) with \(\vec g_t\in\partial f_t(\vec x_t)\).

We formalize the staleness and stability properties mentioned in the introductory example. Fig.~\ref{fig:lazy-vs-greedy} depicts a simple case that contrasts a lazy and a greedy update in polyhedral and strongly convex domains. In the former, large normal cones at the vertices cause a stall in the lazy update, while the greedy update still moves. In the latter, cones reduce to rays, yet the same perturbation induces a smaller displacement when the unconstrained iterate is farther from the boundary.

\begin{figure}[t]
    \centering
    \includegraphics[width=0.99\linewidth]{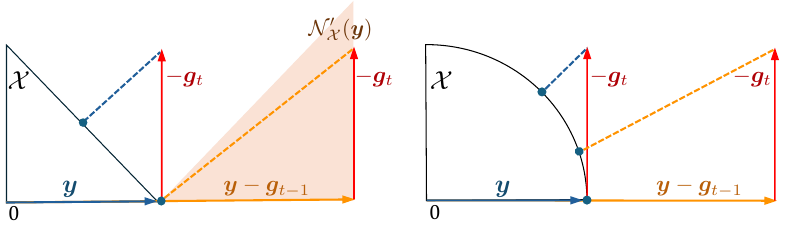}
\caption{Geometric intuition for the effect of lazy vs. greedy updates in $\ell_1$ (left) and $\ell_2$ (right) ball domains. 
From $y$, greedy projects $y-g_t$, while lazy projects $y-g_t-g_{t-1}$ ($k=2$, $\sigma=1$). 
Blue dots are projections.}
    \label{fig:lazy-vs-greedy}
    \vspace{-3mm}
\end{figure}

\subsection{Iterate Staleness}
Denote by $\cone(\vec x_t)$ a closed convex cone associated with $\vec x_t$, and define its translation $\cone'(\vec x_t)$ as: 
\begin{align}
\cone'(\vec x_t) \doteq \vec x_t + \cone(\vec x_t) \;=\; \{ \vec x_t + \vec v : \vec v \in \cone(\vec x_t)\}.
\end{align}
\begin{proposition}[Staleness of Lazy Iterates]\label{prop:staleness}
For any $t$ at which the gradient $\vec g_{t}$ satisfies the inclusion
\begin{equation}\label{eq:st1}
     \vec y_{t+1}=\vec x_{t-n_t} -\frac{1}{\sigma} \vec g_{t-n_t:t} \;\in\; \cone'(\vec x_t),
\end{equation}
for some $0 < n_t\leq k-1$, we have that
\[\|\vec x_{t+1} - \vec x_{t}\| = 0.\]
\end{proposition}

\begin{proof}
From the first order optimality conditions of the Euclidean projection, and the uniqueness of the projection, we have that for any $\vec z$:
\begin{align}
    \label{eq:st2}
    \vec x_t = \Pi( \vec z) \iff \vec z \in \cone'(\vec x_t). 
\end{align}
Comparing \eqref{eq:st1} and \eqref{eq:st2}, we conclude that 
\begin{align}
    \vec x_t = \Pi\big(\vec x_{t-n_t} -\frac{1}{\sigma}\vec g_{t-n_t:t}\big) = \vec{x}_{t+1}.
\end{align}
The last equality by \ref{k-lazy}'s update rule.
\end{proof}
\textbf{Remarks.}
The condition in \eqref{eq:st1} becomes more likely as the laziness parameter $k$ increases. In the lazy case ($k>1$), the newest gradient $\vec g_t$ is added to an \emph{existing aggregation} of up to $k-1$ previous gradients. Its individual influence is therefore dampened; the single most recent gradient is less likely to change the aggregate direction enough to \emph{leave} the cone $\cone'(\vec x_t)$. As a result, the iterate is more likely to become ``stale'' ($\vec x_{t+1} = \vec x_t$), as stated in the proposition.

In contrast, for the greedy case ($k=1$), the update depends entirely on the single, most recent gradient. For the iterate to remain unchanged, we would need $\vec x_{t} - \frac{1}{\sigma} \vec g_{t} \in \cone'(\vec x_t)$, which can be a much stricter condition than \eqref{eq:st1} with a general $k$. This illustrates the ``laziness advantage'': a state that includes a primal point \emph{plus} accumulated gradients is less sensitive to the variance of a single new gradient than a state consisting of \emph{only} a bare primal point.

\subsection{Iterate Stability}
\begin{proposition}[Stability of Lazy Iterates]
    \label{prop:stability}
    For any $t$ at which the gradient $\vec g_{t}$ satisfies the inequality:    
    \begin{align}
        \hspace{-8mm}\|\vec y_{t+1}\| = \|\vec x_{t-n_t} - \frac{\vec g_{t-n_t:t}}{\sigma} \| \geq R+\frac{\alpha G\,n_t+\|\vec g_t\|}{\sigma}, 
        \label{eq:prop-2-cond}
    \end{align}
    for  some $0 <n_t\leq k-1, \alpha\in[0,1]$, and assuming an $\ell_2$-ball domain, we have that 
\begin{align}
    \label{eq:stability1}
    \hspace{-12mm}\|\vec x_{t+1} - \vec x_t\|
    &\le
    \frac{R\,G}{R\sigma + \alpha\,G\,{n_t}}.
\end{align}
That is, when the unconstrained iterate is outside $\X$,\footnote {Otherwise, the projection is not active, and the lazy and greedy states $\vec y_{t+1}$ coincide. The prop. focuses on the interesting case where they differ and affect switching.}and gradients have some correlation parametrized by $\alpha$, the movement is further attenuated. 
\end{proposition}
\begin{proof}
Note that for $\|\vec y\|>R$, projection onto the Euclidean ball is
$\Pi(\vec y)=R\,\vec y/\|\vec y\|$. This map is differentiable on
$\{\vec y:\|\vec y\|>R\}$, and its Jacobian has operator norm
$\|D\Pi(\vec y)\|_{\mathrm{op}}=R/\|\vec y\|$
(shown in Appendix~\ref{app:contractive-property}).
By the mean value inequality for vector mappings,
\begin{align}
&\|\vec x_{t+1} - \vec x_t\|
=
\|\Pi(\vec{y}_{t+1})-\Pi(\vec{y}_{t})\|
\notag\\
&\le
\sup_{\vec z \in [\vec y_t,\vec y_{t+1}]}
\|D\Pi(\vec z)\|_{\mathrm{op}} \,\|\vec{y}_{t+1}-\vec{y}_{t}\|.
\label{eq:radial-contraction-pre}
\end{align}

Next, noting that $t+1-n_{t+1}=t-n_t$, we have
\begin{align}
\label{eq:prop-p1}
\|\vec y_{t+1} - \vec y_t\| = \frac{\|\vec g_t\|}{\sigma} \le \frac{G}{\sigma}.
\end{align}

Moreover, for any $s\in[0,1]$,
\begin{align}
\|\vec y_t+s(\vec y_{t+1}-\vec y_t)\|
\ge
\|\vec y_{t+1}\|-\|\vec y_{t+1}-\vec y_t\|
\\
\ge
\|\vec y_{t+1}\|-\frac{\|\vec g_t\|}{\sigma}
\ge
R+\frac{\alpha G n_t}{\sigma}, \label{eq:prop-p2}
\end{align}
where the first inequality is the reverse triangle inequality and the last
uses~\eqref{eq:prop-2-cond}. Hence the entire segment
$[\vec y_t,\vec y_{t+1}]$ lies outside the ball, and therefore
\[
\sup_{\vec z \in [\vec y_t,\vec y_{t+1}]}\|D\Pi(\vec z)\|_{\mathrm{op}}
=
\frac{R}{\min_{\vec z \in [\vec y_t,\vec y_{t+1}]}\|\vec z\|}.
\]
Substituting into~\eqref{eq:radial-contraction-pre} gives
\begin{equation}
\|\vec x_{t+1} - \vec x_t\|
\le
\frac{R\,\|\vec y_{t+1}-\vec y_t\|}{
\min_{\vec z \in [\vec y_t,\vec y_{t+1}]} \|\vec z\| }.
\label{eq:radial-contraction}
\end{equation}
Using \eqref{eq:prop-p1} and the lower bound on
$\min_{\vec z \in [\vec y_t,\vec y_{t+1}]} \|\vec z\|$ implied by \eqref{eq:prop-p2}
yields the stated bound.
\end{proof}
\textbf{Remark.}
For $k>1$, the bound in~\eqref{eq:stability1} makes explicit the extra attenuation gained within each lazy phase through the $1/n_t$ factor, which reduces the switch size. For example, greedy OGD ($k{=}1$) maintains switches of order $G/\sigma$, whereas in the fully lazy case ($k{=}T$) they shrink to order $G/(\sigma+t)$, which is much stronger.
We note that our use of the $\ell_2$-ball domain is for notational simplicity. The strict contraction property in~\eqref{eq:radial-contraction} holds when projecting for any strongly convex domain, albeit with different constants.

Props.~\ref{prop:staleness} and \ref{prop:stability} formalize the structure exploited by laziness: \eqref{eq:st1} shows that if the unconstrained iterate $\vec y_{t+1}$ remains inside the cone, the update halts; \eqref{eq:prop-2-cond} shows that when $\|\vec y_{t+1}\|$ is large, movement is attenuated. Both effects strengthen as $k$ increases. Appendix~\ref{app:variance-based-sc} further ties switching cost to the deviation of $\vec g_t$ from its running mean, without domain or sequence assumptions. Next, we show that these stability gains are compatible with order-optimal dynamic regret when $k$ is chosen according to the lower bound.

\subsection{Regret Analysis}
\begin{theorem}
\label{thm:DR}
For any sequence of convex losses and comparators, the dynamic regret with switching cost of \ref{k-lazy} satisfies
\begin{align}
    &\R_T \leq \sumT \min\big( G\delta_t,\, \frac{G^2}{2\sigma}\big) + (2R\sigma+kG)P_T 
    \\
    &\quad + \frac{\sigma R^2}{2} + \sumT \min(\delta_t, \frac{G}{\sigma}), 
\end{align}
Moreover, setting\footnote{This choice of $k$ assumes $P_T{>}0$. Otherwise, we may take the largest admissible phase length, $k=T$ (i.e., recover \ref{eq:lazy-gd}).}
\begin{align}& \sigma = \sigma^{\star}\doteq\sqrt{(G^2{+}2G)T/(4RP_T{+}R^2)}, \\&k = \lfloor c\sqrt{2RT/P_T}\rfloor,
\end{align} 
for some $c\geq1$, gives $\R_T \;=\;\mathcal{O}\big(\sqrt{(P_T+1)T}\big).$
\end{theorem}

\textbf{Remarks.} 
The general bound in terms of $(k,\sigma)$ is a result of key lemmas that we present hereafter. 
Ideally, we would optimize $(k,\sigma)$ to minimize this upper bound. However, this is impractical here: the dependence of $\delta_t$ on $k$ is domain-specific and lacks a closed form. Instead, we exploit two key facts established earlier: $(i)$ the switching cost is decreasing in $k$ (from the previous propositions), and $(ii)$ the lower bound restricts $k$ to at most $\Theta(\sqrt{T/P_T})$ to retain optimality. We therefore fix $k$ at this threshold, i.e., use the largest admissible laziness slack. This allows us
to focus only on the $1/\sigma$ branch of the $\min$. Then, we obtain a bound that is convex in $\sigma$, optimizing over $\sigma$ then yields the claimed rate.

The starting point of the proof is the \emph{dynamic} FTRL lemma of~\citet[Lem.~4.1]{pruning-icml}, adapted to include switching cost.  
\begin{lemma}
Let $\{f_{t}(\cdot), \vec u_t\}_{t=1}^T$ be an arbitrary set of functions and comparators in $\mathcal{X}$, respectively. Consider the extended value functions $\bar f_t(\cdot) = f_t(\cdot) + I_\X(\cdot)$ and let $r(\cdot)$ be a strongly convex regularizer such that 
\[
\vec{x}_{t+1} \doteq \argmin_{\vec x} h_{0:t}(\vec{x}) 
\] is well-defined, where $h_0(\vec{x}) \!\doteq\!{I}_{\mathcal{X}}(\vec x)+r(\vec x)$,  $\forall t \geq 1$:
\begin{align}
       h_t(\vec{x}) \doteq \dtp{\vec p_t}{\vec x} \quad \vec p_t \in \partial \bar f_t (\vec x_t).
\end{align}
Then, the algorithm that selects the actions $\vec{x}_{t+1}, \forall t $ achieves the following dynamic regret bound:
    \begin{align}
    &\mathcal{R}_T \leq \sum_{t=1}^T\big(\overbrace{ h_{0:t}(\vec x_t) - h_{0:t}(\vec x_{t+1}) }^{(\mathbf{I})}\big) \,+\sumTO\overbrace{\|\vec x_{t+1} - \vec x_t\|}^{(\mathbf{II})} 
    \\[-3ex]
    &\qquad +\sum_{t=1}^{T-1} \big(\overbrace{h_{0:t}(\vec u_{t+1}) - h_{0:t}(\vec u_{t})}^{(\mathbf{III})}\big) 
    \ +\  r(\vec{u}_1).
\end{align}
\end{lemma}
Note that the lemma's update matches the FTRL form of \ref{k-lazy}, in \eqref{eq:k-lazy-update-FPRL2}, with $r(\cdot) = \nicefrac{1}{2}\|\cdot\|^2$.

This dynamic regret decomposition of FTRL yields three terms: 
($\mathbf{I}$) the penalty for not knowing $f_t(\cdot)$ at decision time, 
($\mathbf{II}$) the switching cost, and 
($\mathbf{III}$) the penalty from comparator shift, which is mainly driven by the state's size $\|\vec p_{1:t}\|$. $r(\vec u_1)$ is the minimum regularization penalty (i.e., when $P_T=0$). These terms can be individually bounded in a way that leads to the result in Thm.~\ref{thm:DR}, as the next lemma states.
\begin{lemma}
\label{lem:supporting}
Under the Lipschitzness of the loss functions and compactness of the domain, \ref{k-lazy} iterates guarantee the following bounds on each part of the dynamic FTRL lemma for any $t$:
\vspace{-1mm}
\begin{itemize}
\renewcommand\labelitemi{}
  \item $\mathbf{(I)} \le \min\!\big(G\delta_t,\, \tfrac{G^2}{2\sigma}\big)$,
  \item $\mathbf{(II)} \le \tfrac{G}{\sigma}$,
  \item $\mathbf{(III)} \le (2R\sigma + kG)\,\dt$.
\end{itemize}
\end{lemma}
The proof of each part is deferred to Appendix~\ref{app:proofs} (Lemma \ref{lem:supporting}.a--c). 
The arguments rely on standard tools from FTRL analysis, which we reuse thanks to Thm.~\ref{thm:equivalence}. However, an extra difficulty arises from the non-standard $\vec g_t^I$ term, whose behavior depends on the slack $k$, requiring a refinement of the classical analysis.

\section{LEARNING THE OPTIMAL SLACK $k$}
Thm.~\ref{thm:DR} assumes optimal tuning of the rate $\sigma$ and slack $k$, requiring knowledge of $P_T$. While we can use an upper bound on it, achieving adaptivity to all comparators requires the unknown $P_T$. We adopt the ensemble framework, which discretizes the $P_T$ space and uses a meta-learner to find the near-optimal choice. For simplicity, we impose the convention $k = \max(1,\lfloor\tfrac{2R}{G}\sigma\rfloor)$. Since the optimal values of both parameters is $\Theta(\sqrt{T{/}P_T})$, this restriction affects only constant factors. The advantage is that the search can be carried out solely over $\sigma$, while implicitly tuning $k$.  

From the bound $P_T \leq 2R(T-1)$, the optimal regularization $\sigma^\star$ is guaranteed to lie in the interval
\begin{align}
    \sqrt{\tfrac{(G^2+2G)T}{R^2}}
    \;\geq\; \sigma^\star
    \;\geq\; \sqrt{\tfrac{(G^2+2G)T}{8R^2T-8R^2}}.
\end{align}
Let $B \doteq \sqrt{\tfrac{(G^2+2G)T}{R^2}},$
and construct the grid:
\begin{align}
    \mathcal H \doteq \Bigl\{ \tfrac{B}{2^{i-1}} : i=1,\dots,N \Bigr\}, 
    \, N \doteq \Bigl\lceil \tfrac{1}{2}\log(8T-7)\Bigr\rceil+1.
\end{align}
Consider the exponent $s^\star \doteq \log\!\big(\sqrt{1+4P_T/R}\big)+1$ and its floor $s \doteq \lfloor \log\!\big(\sqrt{1+4P_T/R}\big)\rfloor+1$. Then,
\begin{align}
    \frac{B}{2^{s-1}} \,\geq\, \frac{B}{2^{s^\star-1}} \,\geq\, \frac{B}{2^s},
    \;\Rightarrow\; 
     \sigma^s \doteq \tfrac{B}{2^{s-1}} \,\geq\, \sigma^\star \;\geq\; \tfrac{B}{2^s}.
\end{align}
By construction, $\sigma^s \in \mathcal H$, and hence the candidate set always contains a near-optimal choice.
\begin{theorem}
For any comparator sequence, running the meta-learner of \cite[``SAder'']{zhang2021revisiting} over a set of \ref{k-lazy} experts from $\mathcal H$, guarantees
\[
    \R_T = \mathcal{O}\big(\sqrt{(P_T+1)T}\big).
    \vspace{-3mm}
\]
\end{theorem}
\paragraph{Remarks.}  
While this bound is known for ensembles of greedy OGD learners~\citep{zhang2021revisiting}, our construction differs in a key respect: the experts vary not only in their learning rate but also in the laziness slack, i.e., how many past gradients are aggregated \emph{and} how they are weighted in the \ref{k-lazy} update. This design allows our ensemble to exploit the structural advantages captured by Propositions~\ref{prop:staleness} and~\ref{prop:stability}. Moreover, if one were to optimize $k$ and $\sigma$ independently, the search would expand to their Cartesian product; by tying them together, we avoid this extra complexity while still achieving order-optimal guarantees. We detail the proof of this result in Appendix~\ref{app:meta}.

\section{RELATED WORK}
SOCO admits two information models: a \emph{look-ahead} setting where the cost is known before acting, usually studied through \emph{competitive ratio} (CR) metric, and a \emph{fully online} setting where the cost is revealed only after the decision, usually studied through regret $\R_T$.

\paragraph{Look-ahead}
An important result in the look-ahead category is Online Balanced Descent (OBD)~\citep{OBD} and its follow-ups. OBD projects the previous iterate onto a carefully chosen level set of the \emph{known} hitting cost to balance movement and hitting costs. This attains dimension-free CR, and its variants attain an optimal CR for specific families of costs ~\citep{R-OBD}. OBD also admits dynamic-regret guarantees with switching costs, but these hold only against a pre-fixed upper bound on $P_T$. The guarantees of OBD family were further generalized to models with look-ahead \emph{window}~\citep{lin2020online} with possibly imperfect predictions ~\citep{rutten2023smoothed}. Because these methods have access to the current hitting cost, they differ from our lazy update, which targets regret guarantees in the online setting.

\paragraph{Fully online.}  
In this setting, \citet{li2020leveraging} studied SOCO with prediction windows under smoothness and strong convexity assumptions. In contrast, \citet{zhang2021revisiting} achieved optimal dynamic regret with switching-cost guarantees for general convex functions,\footnote{They also develop models for the look-ahead setting.} which, along with its generalization to dynamic regret over any sub-interval~\citep{zhang2022smoothed}, represents the current state of the art in both CR and dynamic regret.  
Their analysis relies on an ensemble of \ref{eq:gd} base learners with different $\sigma$ but fixed $k$ ($k{=}1$ for all learners), reinforcing the common view that gradient accumulation (dual-averaging) cannot yield dynamic-regret guarantees. Our work shows that a \emph{rationed} form of accumulation does achieve such guarantees, and can indeed be instantiated within the same ensemble framework.  
Overall, dual-averaging and related lazy methods remain unexplored in SOCO, whether under look-ahead or fully online models, and under either CR or dynamic-regret criteria.

\paragraph{FTRL/OMD interplay.}  
\ref{eq:gd} is typically identified with Online Mirror Descent (OMD) using the squared $\ell_2$ norm as the distance-generating function (see, e.g., \citet[Sec .~6.2]{orabona2021modern}). In contrast, \ref{eq:lazy-gd} corresponds to FTRL with the same function as a regularizer (see, e.g., \citet[Sec.~3.2]{mcmahan-survey17}). 
More recently, \cite{jacobsen2022parameter} introduced ``centered OMD'', a method that can reproduce both the lazy and greedy forms of OGD. However, their work does not investigate nor optimize the interpolation between these two extremes. 
\cite{pruning-icml} 
introduced the pruning perspective as a unifying lens for greedy and lazy updates, showing its role in recovering dynamic-regret guarantees.
However, they do not incorporate laziness slack, which is the quantity we identified, and optimized, as the key driver of low switching cost in SOCO. 

\textbf{Other ``lazy'' forms.}  
The term ``lazy'' has also appeared in the literature with a different meaning. In the related work of \cite{sherman2021lazy} and \cite{agarwal2023differentially}, a method is called lazy if the expected \emph{number} (not magnitude) of switches is bounded. Their construction relies on FTRL, augmented with a \emph{lazy sampling} scheme that couples consecutive decisions maximally, thereby ensuring such switching guarantees. However, these works restrict attention to the weaker static regret metric. Moreover, since our method naturally operates in lazy phases, it is in fact possible in principle to combine \ref{k-lazy} with their lazy-sampling scheme, yielding the same bounded switching number within each phase, while also having dynamic regret guarantees. Finally, we discuss in Appendix~\ref{app:related-work} the related framework and algorithms of ``OCO with memory''\citep{anava2015online}. 

\section{CONCLUSION}
We introduced \textsc{$k$-lazyGD}, a \emph{partially}-lazy learner via an FTRL–pruning view. It achieves minimax dynamic regret while partially retaining the switching behavior of lazy methods, with a matching lower bound. Promising extensions include non-Euclidean and adaptive-rate schemes, incorporating predictions, 
and competitive ratio or adaptive regret analysis.

\section*{Acknowledgment}
This work was supported by the Dutch National Growth Fund project ``Future Network Services'', and by the European Commission through Grants No. 101139270 “ORIGAMI” and No. 101192462 “FLECON-6G”.

\bibliography{References.bib}

\clearpage
\appendix
\thispagestyle{empty}
\onecolumn
\aistatstitle{Partially Lazy Gradient Descent for Smoothed Online Learning \\
Supplementary Materials}

\section{SUMMARY OF \textsc{$k$-lazyGD} FORMS}
\label{app:summary}
We summarize the two formulations of the \textsc{$k$-lazyGD} algorithm, detailed in Algorithms \ref{alg:k-lazy-proj} and \ref{alg:k-lazy-ftrl}. Theorem.~$1$ establishes that the projection and FTRL forms produce identical iterates. Hence, for notational simplicity, we omit the distinction between $\vec{x}_{t}$ and $\hat{\vec{x}}_{t}$ below.  

\paragraph{Projection form.}  
The $k$-lazy update can be expressed directly in terms of projected gradient steps:
\begin{equation}
\tag{\textsc{$k$-lazyGD}}
\label{k-lazy-app}
\begin{aligned}
    & \vec x_{t+1} = \Pi\left(\vec x_{t-n_t} - \frac{\vec g_{t-n_t:t}}{\sigma}\right) \\
    & n_t = (t-1) \!\!\!\mod (k)
\end{aligned}
\end{equation}
where $\Pi$ denotes Euclidean projection onto $\X$.
\begin{algorithm}[h!]
  \caption{\small{\ref{k-lazy} (projection form)}}
  \label{alg:k-lazy-proj}
  \textbf{Input}: compact set $\X$, regularization $\sigma>0$, laziness slack $k\in\mathbb{N}$, horizon $T$.\\
  \textbf{Output}: $\{\vec x_t\}_{t=1}^T$.
  \begin{algorithmic}[1]
    \STATE Initialize $\vec x_1 =0$
    \FOR{$t=1,2,\ldots,T$}
      \STATE Use action $\vec x_t$.
      \STATE \textit{$f_t(\cdot)$ is revealed}
      \STATE Compute a subgradient $\vec g_t \in \partial f_t(\vec x_t)$.
      \STATE Set the counter $n_t \doteq (t-1)\!\!\!\mod(k)$. (within-phase index)
      \STATE Accumulate the within-phase gradients $\vec g_{t-n_t:t} \doteq \sum_{\tau=t-n_t}^{t}\vec g_\tau$.
      \STATE Form the unconstrained step $\vec y_{t+1} \doteq \vec x_{t-n_t} - \frac{1}{\sigma}\,\vec g_{t-n_t:t}$.
      \STATE Project back to the domain $\vec x_{t+1} = \Pi\!\left(\vec y_{t+1}\right)$.
    \ENDFOR
  \end{algorithmic}
\end{algorithm}

\paragraph{FTRL form.}  
Equivalently, \ref{k-lazy} admits the following Follow-the-Regularized-Leader formulation:
\begin{align}
    \vec x_{t+1} &= \Pi(\vec y_{t+1}), \quad 
    \vec y_{t+1} = \argmin_{\vec x \in \mathbb{R}^d}\, \dtp{\vec p_{1:t}}{\vec x} + \frac{\sigma}{2}\|\vec x\|^2,
    \label{eq:k-lazy-update-FPRL-app}
\end{align}
or, equivalently, as a constrained minimization:
\begin{align}
    \vec x_{t+1} = \argmin_{\vec x \in \X}\, \dtp{\vec p_{1:t}}{\vec x} + \frac{\sigma}{2}\|\vec x\|^2. 
    \label{eq:k-lazy-update-FPRL2-app}
\end{align}

\paragraph{State.}  
The state vector $\vec p_t$ accumulates both the loss subgradients and a possible correction term arising from the normal cone:
\begin{align}
&\vec p_t = \vec g_t + \vec g^{I}_t, \\
&\vec g_t \in \partial f_t(\vec x_t), 
\quad 
\vec g^I_t \in \partial I_\X(\vec x_t) = \cone(\vec x_t).
\end{align}

The correction term implements pruning at the prescribed rate:
\begin{align}
\label{eq:pruning-cond-n-app}
 \vec g^I_t = 
    \begin{cases}
       -\vec p_{1:t-1} - \sigma \vec x_t, & \text{if }(\vec y_{t} \notin \X \ \land\  n_t=0), \\[1ex]
        0, & \text{otherwise.}
    \end{cases}
\end{align}
The above provides a complete and equivalent characterization of the \textsc{$k$-lazyGD} update rule in both projection and FTRL form.
\begin{algorithm}[h]
  \caption{\small{\ref{k-lazy} as FTRL with pruning (equivalent to Alg. \ref{alg:k-lazy-proj})}}
  \label{alg:k-lazy-ftrl}
  \textbf{Input}: compact set $\X$, regularization $\sigma>0$, laziness slack $k\in\mathbb{N}$, horizon $T$.\\
  \textbf{Output}: $\{\vec x_t\}_{t=1}^T$.
  \begin{algorithmic}[1]
    \STATE Initialize $\vec x_1 \in \X$, set $\vec p_{1:0} \doteq \vec 0$.
    \FOR{$t=1,2,\ldots,T$}
      \STATE Use action $\vec x_t$.
      \STATE \textit{$f_t(\cdot)$ is revealed}
      \STATE Compute a subgradient $\vec g_t \in \partial f_t(\vec x_t)$.
      \STATE Set the counter $n_t \doteq (t-1)\!\!\!\mod(k)$
      \STATE Compute the unconstrained FTRL center $\vec y_t = -\frac{1}{\sigma}\,\vec p_{1:t-1}$.
      \STATE Set the normal-cone correction $\vec g_t^I$ (pruning rule, cf.~\eqref{eq:pruning-cond-n-app}).
      \STATE Set  $\vec p_t \doteq \vec g_t + \vec g^I_t$ and update the state $\vec p_{1:t} \doteq \sum_{\tau=1}^{t}\vec p_\tau$.
      \STATE Solve the (unconstrained) FTRL subproblem
      \vspace{-3mm}
      \[
        \vec y_{t+1} \doteq \argmin_{\vec x\in\mathbb{R}^d}\;\dtp{\vec p_{1:t}}{\vec x} + \frac{\sigma}{2}\|\vec x\|^2
        \;=\; -\frac{1}{\sigma}\,\vec p_{1:t}.
      \]
        \vspace{-3mm}
      \STATE Project to the feasible set (equivalently, solve the constrained form \eqref{eq:k-lazy-update-FPRL2-app})
    \vspace{-3mm}
      \[
        \vec x_{t+1} = \Pi(\vec y_{t+1})
        \quad\big(\;\text{i.e., }\;\vec x_{t+1} \doteq \argmin_{\vec x\in\X}\;\dtp{\vec p_{1:t}}{\vec x} + \frac{\sigma}{2}\|\vec x\|^2\big).
      \]
    \vspace{-3mm}

    \ENDFOR
  \end{algorithmic}
\end{algorithm}

\section{SWITCHING COST}
\label{app:non-expan}
For both greedy and lazy variants of OGD, the switching cost follows directly from the non-expansiveness of the Euclidean projection (see, e.g., \citet[Thm. 5.4]{beck-book}).  
Formally, for any $\vec u,\vec v \in \mathbb{R}^d$ and convex set $\X$, the projection $\Pi(\cdot)$ satisfies
\begin{align}
\|\Pi(\vec u) - \Pi(\vec v)\| \;\leq\; \|\vec u-\vec v\|.
\label{eq:nonexpansive}
\end{align}

\subsection{\textsc{GD}}
Recall the update rule:
\begin{equation}
\vec{x}_{t+1}
= \Pi\left( \vec{x}_t - \frac{1}{\sigma} \vec{g}_t \right).
\label{eq:gd-app}
\tag{\textsc{GD}}
\end{equation}
Applying~\eqref{eq:nonexpansive} gives
\begin{align}
\|\vec x_{t+1} - \vec x_t\|
&= \Bigl\|\Pi\!\left(\vec x_t - \frac{1}{\sigma}\vec g_t\right) - \vec x_t\Bigr\| \\
&\stackrel{(a)}{\leq} \Bigl\|\vec x_t - \frac{1}{\sigma}\vec g_t - \vec x_t\Bigr\|
= \frac{1}{\sigma}\|\vec g_t\|
\;\leq\; \frac{G}{\sigma}.
\end{align}
The first inequality is by non-expansiveness of $\Pi$, and the second by $G$-Lipschitzness of the losses.

\subsection{\textsc{LazyGD}}
Recall the update rule:
\begin{align}
\vec{x}_{t+1}
= \Pi\left( - \frac{\vec{g}_{1:t}}{\sigma} \right).
\label{eq:lazy-gd-app}
\tag{\textsc{LazyGD}}
\end{align}
Similarly, for the lazy variant, we obtain
\begin{align}
\|\vec x_{t+1} - \vec x_t\|
&= \Bigl\|\Pi\!\left(-\frac{1}{\sigma}\vec g_{1:t}\right) - \Pi\!\left(-\frac{1}{\sigma}\vec g_{1:t-1}\right)\Bigr\| \\
&\stackrel{(a)}{\leq}  \Bigl\|-\frac{1}{\sigma}\vec g_{1:t} + \frac{1}{\sigma}\vec g_{1:t-1}\Bigr\|
= \frac{1}{\sigma}\|\vec g_t\|
\;\leq\; \frac{G}{\sigma}.
\end{align}
Again, the first inequality is due to non-expansiveness of $\Pi$, and the last from Lipschitzness.

\subsection{Discussion}
When the projection is inactive, the inequalities $(a)$ in both \textsc{GD} and \textsc{LazyGD} become equalities. In this case, and assuming $\|\vec g_t\| = G$ for all $t$,  both greedy and lazy variants incur identical per-round movement of $G/\sigma$.

The distinction emerges only when the projection is \emph{active}, the two updates behave differently. For the lazy variant, the term 
\[
\bigl\|\Pi(-\tfrac{1}{\sigma}\vec g_{1:t}) - \Pi(-\tfrac{1}{\sigma}\vec g_{1:t-1})\bigr\|
\]
is often small, or even zero, because the state is represented indirectly through the aggregate $\vec g_{1:t}$. As a result, many consecutive gradient-sum pairs $(\vec g_{1:t-1}, \vec g_{1:t})$ can map to nearly identical primal points, a phenomenon previously noted by \citet[Sec.~6]{mcmahan-survey17}.
In this work, we formalize this observation through our \emph{laziness propositions}: the projection operator of the two expressions (which differ only in a single gradient) can be shown to be multiple-to-one mapping (Prop. \ref{prop:staleness}), or strict contraction (Prop. \ref{prop:stability}). Moreover, in \textsc{$k$-lazyGD}, we partially recover this stabilizing effect by introducing controlled pruning: the switching term now takes the form:
\[
\Bigl\|\Pi\!\left(\vec x_{t-n_t} - \tfrac{1}{\sigma}\vec g_{t-n_t:t-1}\right)
- \Pi\!\left(\vec x_{t-n_t} - \tfrac{1}{\sigma}\vec g_{t-n_t:t}\right)\Bigr\|.
\]
This design ensures that while movements remain bounded by the same $G/\sigma$ worst-case limit, the effective switching cost is often smaller in practice, reflecting the same suppression that characterizes fully lazy methods.

Finally, it is worth noting that when projection operator is inactive $\forall t$, \ref{k-lazy} produces that same iterates for all $k$, implying the equivalence between \eqref{eq:gd} and \eqref{eq:lazy-gd} (recall the constant learning/regularization rate assumption).

\section{MORE DETAILS ON THE ILLUSTRATIVE EXAMPLES}
\label{app:more-details}
We provided two stylized examples in the introduction, and we detail the setup below. 

\noindent\textbf{Example $(i)$.} We consider a 2-dimensional test with horizon \(T=101\). The cost sequence \(g_1,\dots,g_T\in\mathbb{R}^2\) is axis-aligned and unit-norm: \(g_t=(-1,0)\) for \(t=1,\dots,12\), and for \(t=13,\dots,101\) the costs cycle with period four through \((1,0),(0,1),(-1,0),(0,-1)\), producing a periodic drift with direction flips. Decisions \(\vec{x}_t\) lie in the \(\ell_1\)-ball \(\{\vec{x}\in\mathbb{R}^2:\|\vec{x}\|_1\le 1\}\). The regularization rate is fixed to \(\sigma=\sqrt{t}\) for both algorithms. For the actions snapshot, we use a fixed \(\sigma=\sqrt{T}\) instead to better capture the four actions.
\begin{figure}[h]
    \centering
    \includegraphics[width=0.65\linewidth]{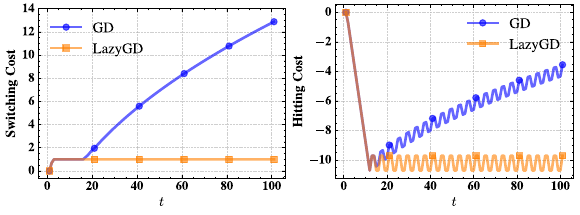}
    \caption{Switching and hitting cost in Example $(i)$.}
\end{figure}

\noindent\textbf{Example $(ii)$.} We consider a 2-dimensional test with horizon \(T=101\). The cost sequence \(g_1,\dots,g_T\in\mathbb{R}^2\) is axis-aligned: \(g_t=(-1,0)\) for \(t=1,\dots,11\), and for \(t=12,\dots,101\) we set \(g_t = (-1,(-1)^t)\), i.e.\ the second coordinate alternates in sign producing \((-1,1),(-1,-1),\dots\). This yields a piecewise-constant initial segment followed by rapid vertical oscillations with abrupt sign flips. Decisions \(\vec{x}_t\) lie in the \(\ell_2\)-ball \(\{\vec{x}\in\mathbb{R}^2:\|\vec{x}\|_2\le 1\}\). Again, the regularization rate is fixed to \(\sigma=\sqrt{t}\) for both algorithms. For the actions snapshot, we use a fixed \(\sigma=\sqrt{T}\) instead to better capture the four actions.

\begin{figure}[h]
    \centering
    \includegraphics[width=0.65\linewidth]{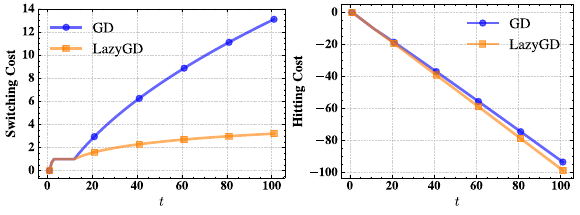}
    \caption{Switching and hitting cost in Example $(ii)$.}
\end{figure}
The message of the two examples is:  Greedy GD ($k=1$) is not necessarily good for SOCO results. In this scenario, stepping in the direction of the most recent gradient causes worse movement \emph{and} hitting cost than aggregating all past gradients.

\section{\textsc{$k$-lazyGD} SIMULATIONS}
We complement our theory with simulations illustrating the behavior of \eqref{k-lazy}\footnote{The code for the experiments can be found on the following repository  \url{https://github.com/Naram-m/k-lazy}.}.  
Our goal is to verify that partial laziness can substantially reduce switching cost without degrading hitting performance.  
In particular, we expect the switching cost to be empirically comparable to \eqref{eq:lazy-gd} and the hitting cost close to \eqref{eq:gd}, leading to overall performance that improves upon the standard SOCO baseline \eqref{eq:gd}. 

For each algorithm, we report three quantities:   
\begin{align}
    &\hspace{-4mm}\text{\emph{switching cost:}}\quad \sum_{t=1}^T \|\vec x_t - \vec x_{t-1}\|,
    \\
    &\hspace{-4mm}\text{\emph{dynamic regret:}} \quad \sum_{t=1}^T \big(f_t(\vec x_t) - f_t(\vec u_t)\big),
    \\
    &\hspace{-4mm}\text{\emph{total regret:}}\;
\sum_{t=1}^T \Big(f_t(\vec x_t) - f_t(\vec u_t)\Big)
+ \sum_{t=1}^T \Big(\|\vec x_t - \vec x_{t-1}\| - \|\vec u_t - \vec u_{t-1}\|\Big)= \R_T-\! \sumTO \|\vec u_t - \vec u_{t-1}\| .
\label{eq:total_regret}
\end{align}
Note that~\eqref{eq:total_regret} differs from the regret $\R_T$ defined in the main text in that it subtracts the comparator’s path length.  
This simplification was adopted in the theoretical analysis for cleaner presentation, since $\R_T$ still upper bounds~\eqref{eq:total_regret}.  
For the numerical results, however, we plot the full expression~\eqref{eq:total_regret}.

\subsection{Shifting Stochastic Sequences}
In this experiment, gradients $\vec g_t$ are normalized to satisfy $\|\vec g_t\|_2 = 1$, so $G \doteq 1$.  
We consider linear costs
\[
f_t(\vec x) \doteq \langle \vec g_t, \vec x \rangle,
\qquad \vec g_t \in [-1,1]^5,\;\|\vec g_t\|_2 = 1,
\]
with decisions constrained to the $\ell_2$ unit ball
\[
\X \doteq \{\vec x \in \mathbb{R}^5 : \|\vec x\|_2 \le 1\}.
\]

\paragraph{Sequence generation.}
Let $T_{\mathrm{phase}} \doteq 4000$, $P \doteq 15$.  
We generate $P$ phases with alternating mean directions $\mu_p \in \{+1,-1\}$ (starting positive), and sample within each phase i.i.d. Gaussian gradients with variance $10$ per coordinate.  
Each sampled row is then $\ell_2$-normalized (i.e., projected Gaussian).
Concatenating the $P$ phases yields a horizon of length $T_{\text{tot}} \doteq P\,T_{\mathrm{phase}}$. This construction is interesting because it is reflective of the Stochastically Extended Adversary (SEA) model: environment \emph{is} stochastic, perhaps with high variance, but with  \emph{adversarial} distribution shifts (see more details on the SEA model in \cite{sarah-between}).

\begin{figure}[h!]
    \centering
    \includegraphics[width=0.5\linewidth]{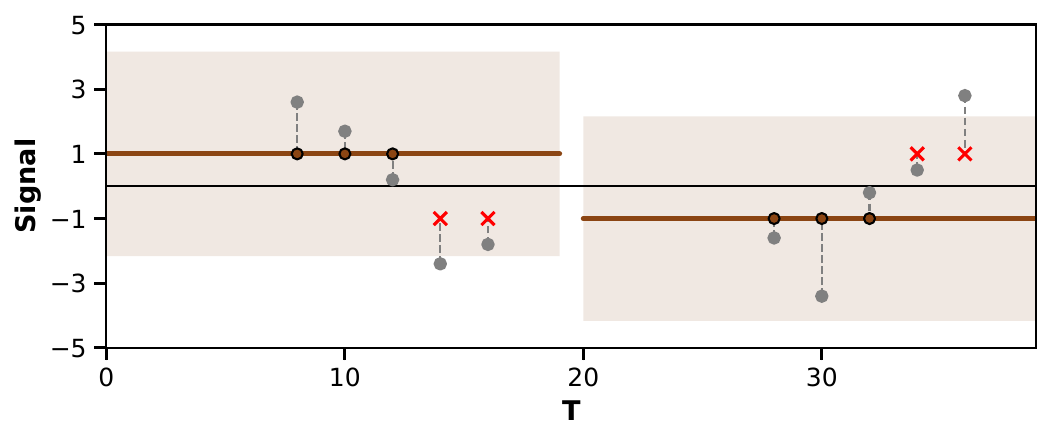}
    \vspace{-4mm}
    \caption{Example cost sequence. $1$ coordinate, $2$ phases of length $20$. Few samples of each phase and their projections on $\{1,-1\}$ are shown. Red `x' is a point of opposite sign to the phase.}
\end{figure}

\paragraph{Comparator.}
The comparator $\{\vec u_t\}$ is piecewise-constant across phases.  
In phase $p$, $\vec u_t$ is fixed to a unit $\ell_2$ vector pointing opposite to the phase mean (equal-mass over coordinates), and switches at phase boundaries; hence the path length $P_T=28$ ($2$ at each switch, with $14$ total switches).

\paragraph{Algorithms.}
We instantiate \textsc{GD}, \textsc{LazyGD}, and \textsc{$k$-lazyGD} with laziness parameters $k \in \{65,150,300,1500\}$.  
All algorithms operate in $\X$ on the same gradient sequence $\{\vec g_t\}$ and use an identical regularization schedule, $
\sigma = \sqrt{\nicefrac{t}{60}} \;\approx\; \sqrt{\nicefrac{t}{2\tau}},
$
where $\tau=30$ is an upper bound on the comparator path length.  
Since the horizon $T$ is unknown, each algorithm simply substitutes the current time $t$ for $T$. Other approaches, e.g., the doubling trick, are possible; the important point is that all methods are treated consistently.  The sole distinction is that \textsc{$k$-LazyGD} varies in $k$; all other settings are kept the same. 
Note that $\sqrt{T/\tau} \approx 50$ (recall $k^\star = \Theta(\sqrt{T/P_T})$), so the first two $k$ values are representative of this range.  
In the experiments, we observe that $k \in \{65,150,300\}$ all outperform the \ref{eq:gd} benchmark, achieving improvements of up to $45\%$.\footnote{If the ensemble framework were employed, it would automatically select the best-performing $k$ for the realized sequence, regardless of the minimax-optimal $k^\star$. While $k^\star$ is a minimax choice, in easier sequences other values of $k$ may perform better, which is precisely the motivation for Sec. $5$.}

\begin{figure}[h]
    \centering
    \includegraphics[width=0.99\linewidth]{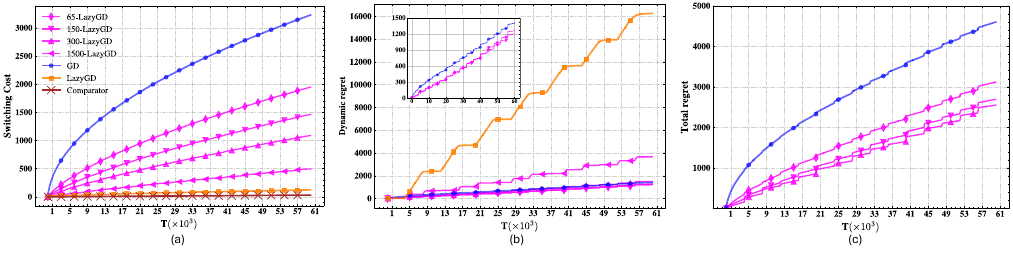}
    \caption{Shifting stochastic sequences. 
    $(a)$ Switching cost, $(b)$ hitting cost, and $(c)$ total regret for \textsc{GD}, \textsc{LazyGD}, and \textsc{$k$-LazyGD} with different $k$.}
    \label{fig:STOCH}
\end{figure}

\subsection{Corrupted Sequences}
We next consider a deterministic counterpart to the stochastic phases above.  
Gradients $\vec g_t$ are normalized to satisfy $\|\vec g_t\|_2 = 1$, so $G \doteq 1$.  
We consider linear costs
\[
f_t(\vec x) \doteq \langle \vec g_t, \vec x \rangle,
\qquad \vec g_t \in [-1,1]^5,\;\|\vec g_t\|_2 = 1,
\]
with decisions constrained to the $\ell_1$ unit ball
\[
\X \doteq \{\vec x \in \mathbb{R}^5 : \|\vec x\|_1 \le 1\}.
\]

\paragraph{Sequence generation.}In each phase, gradients are predominantly aligned with one sign for long stretches (here of length $100$), but are interspersed with short bursts of the opposite sign (length $10$).  
The phase length is $2000$, and a phase is labeled positive or negative according to its dominant sign, while the bursts act as controlled perturbations.  
All gradients are $\ell_2$-normalized, and concatenating $P=10$ alternating phases yields a horizon of length $T_{\text{tot}} \doteq P\,T_{\mathrm{phase}}$.  

This construction is interesting because it isolates the effect of structured, adversarially placed bursts, without the variability introduced by sampling. It therefore serves as a complementary stress test for how the algorithms respond to disruptive fluctuations.  

\begin{figure}[h!]
    \centering
    \includegraphics[width=0.5\linewidth]{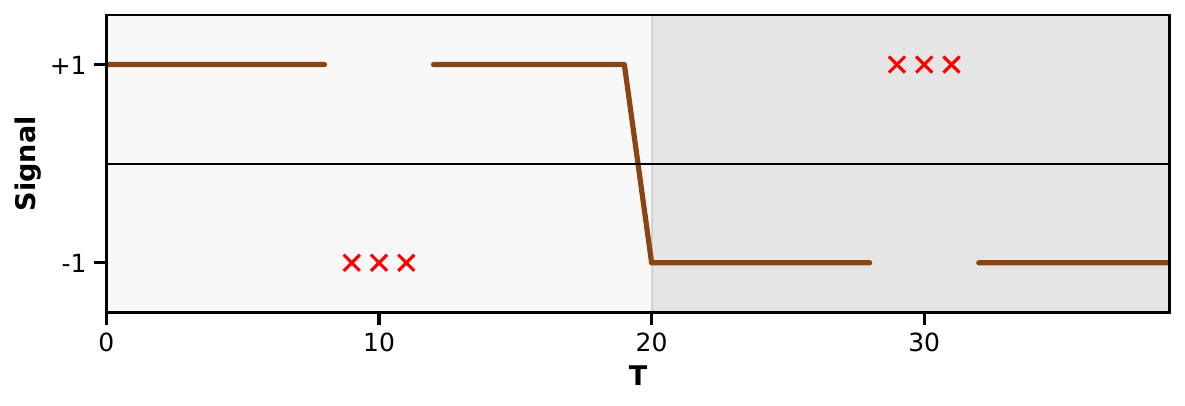}
    \caption{Example cost sequence. $1$ coordinate, 2 phases of length $20$. Red `x' is the corruption burst.}
\end{figure}

\paragraph{Comparator.}
As before, the comparator $\{\vec u_t\}$ is piecewise-constant across phases, always pointing opposite to the phase mean.  
Switches occur at phase boundaries, giving path length $P_T=9\,\frac{2}{\sqrt{5}}$ ($\frac{2}{\sqrt{5}}$ at each of the $P=9$ switches, recall $\vec u$ is unit $\ell_1$ norm here).  

\paragraph{Algorithms.}
We evaluate the same set of algorithms as in the stochastic case: \textsc{GD}, \textsc{LazyGD}, and \textsc{$k$-lazyGD} with $k \in \{50,150,500,1500\}$.  
All methods share the same regularization schedule: $
\sigma = \sqrt{\nicefrac{t}{16}} \;\approx\; \sqrt{\nicefrac{t}{2\tau}}$ (recall $P_T\leq \tau$), and \emph{differ only in the laziness parameter} $k$. 
Note that $\sqrt{T/\tau} \approx 50$ (Recall $k^\star=\Theta(\sqrt{T/P_T})$), and hence the first two $k$ values are representative of this range.

\begin{figure}[h]
    \centering
    \includegraphics[width=0.99\linewidth]{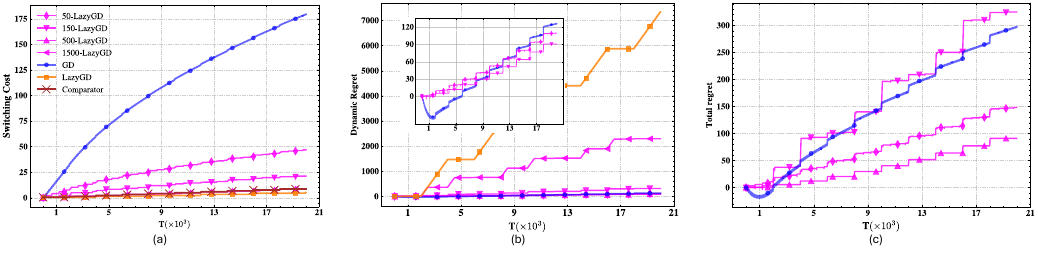}
\caption{Corrupted phases. 
    $(a)$ Switching cost, $(b)$ dynamic (hitting) regret, and $(c)$ total regret for \textsc{GD}, \textsc{LazyGD}, and \textsc{$k$-LazyGD} with different $k$.}    \label{fig:CP}
\end{figure}

\subsection{Takeaway from Experiments}
The simulations highlight three key observations:  
$(i)$ \textsc{$k$-lazyGD} can achieve substantial reductions in switching cost across different choices of $k$ (see parts~$(a)$).
$(ii)$ In many cases, these savings come with little deterioration in hitting cost, and in some cases even yield improvements (parts~$(b)$).
$(iii)$ The reduction in switching mostly dominates the minor loss in tracking, resulting in lower overall regret (parts~$(c)$).
Overall, the results show that controlled laziness offers a favorable trade-off, yielding both stability and responsiveness. 

The advantages of $k$-lazy algorithms are most pronounced in settings with many \emph{undue movements} (movements that add switching cost without improving hitting performance).  
Both experimental scenarios share this property.  
In the stochastic case, chasing individual random gradients induces movement but rarely improves tracking, since subsequent samples tend to cancel out and only the mean direction matters.  
In the corrupted-phase case, reacting to short bursts similarly increases movement, yet the transient gains are offset by the losses incurred when returning to the baseline.  

\paragraph{The worst case for \textsc{$k$-lazyGD}.} Of course, one can construct sequences that are adversarial for $k$-lazy methods, where \emph{every} switch is informative and delaying movement is always worse than immediate movements (this is the lower-bound construction). Though we argue this is less practical than the SEA or corrupted sequences above, we provide this case for transparency. 
The setup is similar to that of corrupted sequences (with $5$ phases of length $300$ each), except there are \emph{no} outliers; each phase is a \emph{pure} sign (and the comparator is of opposite sign). Here $\sigma = \sqrt{\nicefrac{t}{2\tau}}, \tau = 4,\, k=19 = \lfloor\sqrt{T/ \tau}\rfloor$. Note that the final regret amount is still in the order of $\sqrt{2R\,(P_T{+}1)\,T}$.

\begin{figure}[h]
    \centering
    \includegraphics[width=0.9\linewidth]{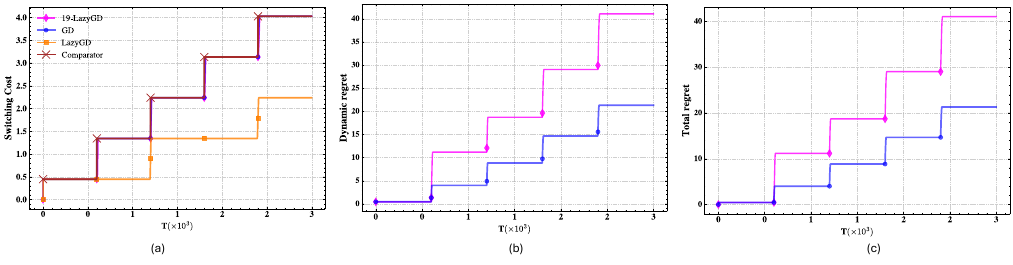}
\caption{Worst-case for \textsc{$k$-lazyGD} methods: every switch is the beginning of a new phase: no outliers due to randomness or deterministically. 
    $(a)$ Switching cost, $(b)$ dynamic (hitting) regret, and $(c)$ total regret for \textsc{GD}, \textsc{LazyGD}, and \textsc{$k$-LazyGD} with different $k$.}    \label{fig:WC}
\end{figure}
Note that the existence of such sequences 
does not undermine the findings: the ensemble framework can already include the greedy update as the special case $k=1$, ensuring robustness across all scenarios. In practice, this means that when undue movements dominate, larger $k$ values yield stability benefits, while in worst-case settings the ensemble naturally falls back to the fully greedy strategy.

\subsection{The Ensemble Framework: \textsc{SAder} with OGD vs. \textsc{$k$-lazyGD} Base Learners}
As discussed earler, the optimal slack parameter $k$ is not known in advance. We therefore learn it online via an ensemble approach. Specifically, we adopt the Smoothed Ader (SAder) meta-learner of \cite{zhang2021revisiting} and use it to aggregate a set of base learners with different slack values.

We compare two ensembles:
$(i)$ the original \textsc{SAder}, which aggregates a set of OGD base learners with different learning rates, and 
$(ii)$ \textsc{SAder-$k$}, which uses the \emph{same} SADER meta-learner (with the same meta learning rate and the same learning-rate set), but replaces each OGD base learner with a $k$-lazyOGD base learner. 

\paragraph{SAder (OGD experts).}
The base learners are OGD algorithms with learning rates
\[
\eta \in \left\{ 
\frac{1}{\sqrt{T}}, 
\frac{2}{\sqrt{T}}, 
\ldots, 
\frac{16}{\sqrt{T}} 
\right\}.
\]

\paragraph{SAder-$k$ (\ref{k-lazy} experts).}
We use the same learning-rate set,
and associate to each $\eta$ a laziness slack
$k = \frac{5\sigma}{2} = \frac{5}{2\eta}$. Thus, each OGD expert is replaced by its corresponding \textsc{$k$-lazyGD} variant.

\subsubsection{Shifting Stochastic Sequences}

\begin{figure}[h]
    \centering
    \includegraphics[width=0.9\linewidth]{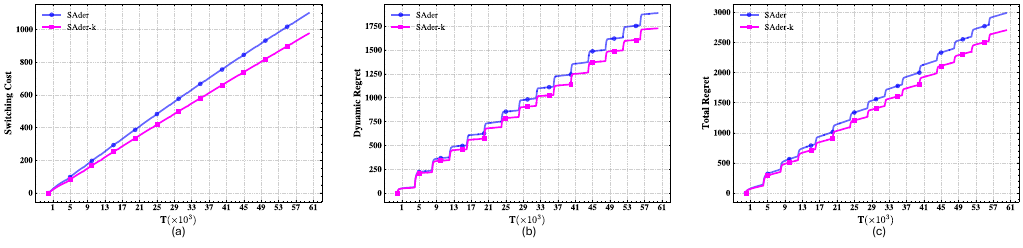}
\caption{Shifting stochastic sequences with ensemble over greedy (\textsc{SAder}) vs. \textsc{$k$-lazy} (\textsc{SAder-k}) learners. 
    $(a)$ Switching cost, $(b)$ dynamic (hitting) regret, and $(c)$ total regret. }    \label{fig:sader-k-stoch}
\end{figure}

We follow the shifting stochastic phases setup from above.
Figure~\ref{fig:sader-k-stoch} reports switching cost, hitting (dynamic) regret, and their sum (total regret).
\textsc{SAder-$k$} improves both components ($\sim 8.5\%$ in hitting costs, and $11.4\%$ in improvement cost).

These gains are consistent with the SEA-like structure of the sequence: within each phase, the variance in gradients induce frequent, often uninformative OGD updates, whereas \textsc{$k$-lazyGD} filters  such fluctuations (by simply aggregating) and reduces unnecessary switching.

\subsubsection{Corrupted Sequences}
\begin{figure}[h]
    \centering
    \includegraphics[width=0.9\linewidth]{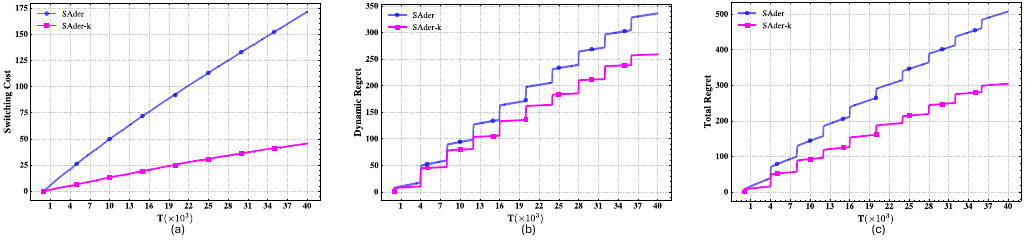}
\caption{Corrupted sequences with ensemble over greedy (\textsc{SAder}) vs. \textsc{$k$-lazy} (\textsc{SAder-k}) learners. 
    $(a)$ Switching cost, $(b)$ dynamic (hitting) regret, and $(c)$ total regret. }    \label{fig:sader-k-det}
\end{figure}
We next follow the corrupted phases setup. Figure~\ref{fig:sader-k-det} summarizes the results.
Here, \textsc{SADER-$k$} yields a notable reduction in switching cost alongside improved hitting performance: ($\sim 23\%$ in hitting costs, and $74\%$ in improvement cost)

The improvement is driven by the ability of \textsc{$k$-lazyGD} experts to ignore short corruption bursts that would otherwise trigger expensive, transient movements in greedy OGD experts; the meta learner can then concentrate weight on these more stable experts, resulting in lower total regret.

\section{MORE DETAILS ON THE LOWER BOUND}
\label{app:LB-details}
The idea behind the lower bound is to construct a sequence that makes any $k$-lazy algorithm behave as if it were fully $T$-lazy.  
This is done by forcing the iterate to return to the feasible set every $k$ steps, thereby preventing pruning during the following $k-1$ steps, and so on. 

Within each such $k$-step interval, and since pruning occurs between intervals, the learner accumulates a loss of at least $-\frac{k}{2}$, while the comparator achieves $-k$ (as long as it still has two switches available).  
Hence, the regret \emph{per interval} is roughly $\frac{k}{2}$, and the total regret is on the order of $\frac{k}{2}\cdot \frac{P_T}{2}$.  
The formal proof in the main text spells out the index arithmetic needed to make this precise.

\subsection{Odd $k$}
When $k$ is odd, we use the modified square shape 
\begin{align}
\label{eq:cost-const-app}
g_t = \begin{cases}
+1, & \text{for } t \in \left[(i-1)k+1, (i-1)k + \frac{k-1}{2}\right] \quad \text{(First half)} 
\\
-1, & \text{for } t \in \left[(i-1)k + \frac{k-1}{2}+1, ik-1\right] \quad \text{(Second half)}
\\0, &\text{for } t = ik \quad \text{(Residual)}
\end{cases}
\end{align}
$\text{for } i = 1, \dots, \lfloor T/k\rfloor .$
This construction maintains the two properties that the lower bound used. Mainly, it is still $k$ periodic: 
That is, $\forall i\in\{1,\dots,\lfloor \frac{T}{k}\rfloor\},j\in\{0,\dots, T\!-\!ik\}$: 
\begin{align}
    \label{eq:sum-base-period-app}
    \sum_{t=j+1}^{j+ik}\!g_t = 0.\; 
\end{align}
Moreover, since the sequence starts with the positive half, the sum over any window is non-negative:
\begin{align}
    \sum_{t=1}^{j} g_t \geq 0,\quad  \forall j\in[T].
    \label{eq:lb-non-negativity}
\end{align}
Moreover, the losses within an active interval differ slightly for the learner and the comparator.  
Specifically, the learner incurs $-\tfrac{k-1}{2}$, while the comparator incurs $-(k-1)$, since the final step contributes zero cost.  
Hence, the regret over an active interval is
\begin{align}
    -\frac{k-1}{2} - \bigl(-(k-1)\bigr) \;=\; \frac{k-1}{2},
\end{align}
which yields the same lower bound for all $k \geq 2$ (noting that $k=1$ corresponds to \textsc{GD}, for which this argument does not apply).

Note that combining our lower bound with the universal one $\Omega(\sqrt{T(P_T+1)})$ from \cite[Thm. 2]{zhang2018adaptive} results in an overall lower bound of: 
\begin{align}
    \label{eq:lower-bound-max}
    \R_T = \Omega\left(\max\left(\sqrt{T(P_T+1)}, k\,P_T \right)\right)
\end{align}

\section{ADDITIONAL RELATED WORK}
\label{app:related-work}
\subsection{SOCO, Adaptive Adversaries, and Policy Regret.}

This subsection connects the SOCO framework to related concepts in the online learning literature, in a way that, to our knowledge, does not appear compiled in one place, especially for the SOCO model.  Our aim is to clarify why the cost is naturally decomposed into hitting and movement terms, why the updates optimize only the hitting part $f_t$, and how this nevertheless yields control over the joint hitting{+}switching objective. We also show that other convex and Lipschitz movement costs can be incorporated.
The discussion below focuses on the fully online (no look-ahead) variant of SOCO.

\paragraph{Absorbing switching into the per-round loss.}
A natural idea is to absorb the movement penalty into the per-round loss (as it stays convex in $\vec x$) and apply a standard OCO algorithm. Concretely, define
\begin{align}
    j_t(\vec x)\;\doteq\; f_t(\vec x)\;+\;\|\vec x-\vec x_{t-1}\|.
\end{align}
If an algorithm guarantees sublinear static regret on $\{j_t\}$, this seems to yield small hitting and switching costs simultaneously. The problem is that the induced benchmark is \emph{misaligned} with the SOCO objective. Indeed,
\begin{align}
\sumT \bigl(j_t(\vec x_t) - j_t(\vec u_t)\bigr)
= \sumT\!\Big( f_t(\vec x_t)-f_t(\vec u_t)\Big)
\;+\; \sumT\!\Big(\|\vec x_t-\vec x_{t-1}\| - \|\vec u_t-\vec x_{t-1}\|\Big),
\label{eq:misaligned}
\end{align}
which compares our switching to the \emph{comparator’s switching measured against our own history} $\vec x_{t-1}$, rather than its history $\vec u_{t-1}$. As first noted in \cite{arora2012online}, this corresponds to a \emph{reactive/non-oblivious} setup: the loss mapping $j_t:\X\to\mathbb{R}$ is \emph{parameterized by the learner’s realized decision} $\vec x_{t-1}$.

In SOCO, the resulting mismatch term is \emph{negative} and may be dropped. That is, this misaligned quantity in \eqref{eq:misaligned} can still be bounded by $\R_T$ similar to what we did in \eqref{eq:total_regret}. For a \emph{single fixed $k$}, such looseness can still suffice to analyze SOCO. However, this approach breaks down in \emph{ensemble} (expert/meta) designs: to tune $\sigma, k$ online, we require the algebraic cancellation between meta-regret and base-expert regret (see the common terms in Theorem.~8, especially the cancellation between \eqref{eq:cancel-1} and \eqref{eq:cancel-2}). That cancellation fails if per-round losses reference the learner’s own trajectory (via $\vec x_{t-1}$) rather than an external benchmark. 

Hence, embedding switching inside $j_t(\vec x)$ is brittle; even when it does not alter a fixed $(\sigma, k)$ analysis, it breaks meta-learning over $(\sigma, k)$. This motivates the standard SOCO practice of analyzing hitting regret against $\{f_t\}$ and \emph{separately} controlling  $\|\vec x_t-\vec x_{t-1}\|$. A parallel route to the same two-part analysis is via policy regret.

\paragraph{Policy regret fixes the benchmark.}
The remedy is to explicitly model dependence on the previous action and study \emph{policy regret} for an $m\!=\!1$ memory loss. Define the bivariate loss
\begin{align}
    j_t(\vec y,\vec x)\;\doteq\; f_t(\vec x)+\|\vec x-\vec y\|, 
    \qquad (\vec y,\vec x)\in \X\times\X,
\end{align}
and evaluate
\begin{align}
\sumT \underbrace{j_t(\vec x_{t-1},\vec x_t)}_{\text{our policy at $t$}}
\;-\;
\underbrace{j_t(\vec u_{t-1},\vec u_t)}_{\text{comparator policy at $t$}}
\;=\;
\sumT\!\Big( f_t(\vec x_t)-f_t(\vec u_t)\Big)
\;+\; \sumT\!\Big(\|\vec x_t-\vec x_{t-1}\|-\|\vec u_t-\vec u_{t-1}\|\Big),
\end{align}
which exactly recovers the SOCO dynamic-regret-with-switching objective. This places us in the \emph{oblivious-with-memory} model: the loss mapping $j_t:\X\times\X\to\mathbb{R}$ is \emph{independent of the learner’s realized decision} (hence non-reactive), and dependence on $(\vec x_{t-1},\vec x_t)$ arises only through \emph{evaluation}, not through the \emph{definition} of $j_t$ after observing the learner’s past. Such memoryful but oblivious adversaries were first handled in \cite{anava2015online}.

\paragraph{Unary relaxation and why we optimize only $f_t$.}
As observed by \cite{anava2015online}, memory losses admit a Lipschitz relaxation. For any loss $l_t$ that is $L$-Lipschitz in its first argument,
\begin{align}
    l_t(\vec y,\vec x)
    \;\le\;
    \underbrace{l_t(\vec x,\vec x)}_{\doteq\;\tilde l_t(\vec x)\; \text{unary form}}
    \;+\; L\,\|\vec y-\vec x\|. \label{eq:unary-relatxation}
\end{align}
This yields the following recipe:
(i) run any OCO method on the unary losses $\tilde j_t(\vec x)$, and
(ii) ensure successive decisions are stable (small $\|\vec x_t-\vec x_{t-1}\|$).
Together, these two properties imply a policy-regret guarantee, hence a SOCO guarantee.

For SOCO, $j_t(\vec y,\vec x)$ is $1$-Lipschitz in $\vec y$,\footnote{By the triangle inequality, $\big|\,\|\vec x-\vec y\|-\|\vec x-\vec y'\|\,\big|\le \|\vec y-\vec y'\|$ for fixed $\vec x$.}
and its unary form is
\[
j_t(\vec x,\vec x)
\;=\; f_t(\vec x)+\|\vec x-\vec x\|
\;=\; f_t(\vec x).
\]
Therefore, it suffices to run an OCO algorithm on $f_t$ with appropriate (static/dynamic) regret guarantees, provided the algorithm also ensures small movement. 

This perspective clarifies why SOCO algorithms, classical GD, dual-averaging / lazy-GD, and our partially-lazy variants, use only gradients of $f_t$: the hitting part is handled by OCO guarantees, while the switching part is handled by stability (via non-expansiveness and regularization). Our $k$-lazy design strengthens this effect, as quantified in Propositions~3 and~4 in the main text, and \ref{prop:variance} presented later in the appendix.

\paragraph{On other switching-cost forms.}  
From \eqref{eq:unary-relatxation}, we can see that the fixed $\ell_2$ norm can be replaced by any other norm (not necessarily Mahalanobis norms $\|\cdot\|_{A_t}$), provided the movement cost is Lipschitz with respect to that norm and the constant $L$ is adjusted accordingly\footnote{Convexity is also required for the meta learning framework.}. Moreover, by properties of Euclidean projection, the switching cost remains bounded in the $\ell_2$ norm, which in turn controls all other norms by equivalence of norms.

The setting in the main paper corresponds to a fixed movement coefficient \(\lambda \|x_{t+1}-x_t\|\) with \(\lambda=1\).\footnote{Equivalently, one may view this as a switching-cost function that is \(1\)-Lipschitz with respect to the underlying norm. More generally, the same argument applies to any fixed \(L\)-Lipschitz switching penalty, with the constants adjusted by \(L\).} The analysis extends to any fixed \(\lambda>0\): the movement term in the dynamic FTRL decomposition is simply multiplied by \(\lambda\), and optimizing the regularization parameter accordingly introduces the corresponding \(\sqrt{\lambda}\) dependence in the final bound.  A concurrent work~\citep{esposito2026parameter} studies a complementary setting with time-varying movement coefficients \(\{\lambda_t\}_{t=1}^T\) and proves an adaptive bound of the form $\tilde{\mathcal{O}}(\sqrt{(1+P_T)(T+\sum_t\lambda_t)})$. Their results are developed for unconstrained domains and are therefore comparator-adaptive. Here, we focus on the different question of the possibility of (partial) laziness and the impact thereof on the switching cost in bounded domains, where the projection operator interacts favorably with a (partially) lazily aggregated state, an aspect not addressed in their framework.

\subsection{Relation to other Phased Updates}

The phased structure of the update rule in~\textsc{$k$-lazyGD} is reminiscent of blocking methods, which also operate in phases~\citep{merhav2002sequential, chen2020minimax}. These methods adopt the ``blocking argument'': the time horizon $T$ is partitioned into $S$ (or $k$) equally sized blocks, and the cumulative loss of each block is treated as a single loss function for a new OCO problem with only $S$ rounds. While conceptually analogous, the implementation is critically different. The blocking strategy implies a piecewise-constant approach where \emph{a single decision} is made and held constant for the duration of each block, with an update occurring only at the block boundaries. In contrast, the update in~\textsc{$k$-lazyGD} generates a new decision vector at every time step. This decision follows the regular FTRL update, but adjusted to be initialized differently at the beginning of each phase.

A second point of contrast is with the literature on OCO with delayed feedback~\citep{joulani2016delay, flaspohler2021online}, which also use ``stale'' information from $t-d_{t}$ (or $t-n_{t}$) where the $d_t$ is an external, potentially adversarial, imposed delay. In the delayed-feedback setting, the gradient $\vec g_t$ is not available at the end of round $t$ but arrives at some future time $t+d_t$. The update in~\textsc{$k$-lazyGD} operates under a different paradigm. It exists within the standard, non-delayed OCO framework where the gradient $\vec g_t$ is revealed immediately. The ``laziness'' or ``delay'' inherent in the~\textsc{$k$-lazyGD} update is not an external constraint to be overcome but rather a deliberate, deterministic \emph{algorithmic mechanism}. The algorithm has access to all gradients $\vec g_1, \dots, \vec g_t$ when computing $\vec x_{t+1}$ but chooses to anchor its update to a past iterate $\vec x_{t-n_t}$ and use an accumulation of recent gradients thereafter in order to avoid over-reactivity to noisy or anomalous gradients.

\section{\textsc{$k$-LazyGD} ANALYSIS}
\label{app:proofs}
\subsection{Proof of Lemma $6$}

\paragraph{Lemma $\mathbf{6}$.}
Let $\{f_{t}(\cdot), \vec u_t\}_{t=1}^T$ be an arbitrary set of functions and comparators within $\mathcal{X}$, respectively. Let $r(\cdot)$ be strongly convex regularizer such that 
\[
\vec{x}_{t+1} \doteq \argmin_{\vec x} h_{0:t}(\vec{x}) 
\] is well-defined, where $h_0(\vec{x}) \!\doteq\!{I}_{\mathcal{X}}(\vec x)+r(\vec x)$,  $\forall t \geq 1$:
\begin{align}
       h_t(\vec{x}) \doteq \dtp{\vec p_t}{\vec x} \quad \vec p_t \in \partial \bar f_t (\vec x_t).
\end{align}
Then, the algorithm that selects the actions $\vec{x}_{t+1}, \forall t $ achieves the following dynamic regret bound:
    \begin{align}
    &\mathcal{R}_T \leq \sum_{t=1}^T\big(\overbrace{ h_{0:t}(\vec x_t) - h_{0:t}(\vec x_{t+1}) }^{(\mathbf{I})}\big) \,+\sumTO\overbrace{\|\vec x_{t+1} - \vec x_t\|}^{(\mathbf{II})} +\sum_{t=1}^{T-1} \big(\overbrace{h_{0:t}(\vec u_{t+1}) - h_{0:t}(\vec u_{t})}^{(\mathbf{III})}\big) 
    \ +\  r(\vec{u}_1).
\end{align}

\begin{proof}
We start from a quantity that is the dynamic regret w.r.t. the $h_t(\cdot)$ functions. Then we decompose it further. 
\begin{align}
    \sumT& h_t(\vec x_t) - \sumT h_t(\vec u_t) 
    \\
    &= \sumT \left(h_{0:t}(\vec x_t) - h_{0:t-1}(\vec x_t)\right) - \left( \sumT \left(h_{0:t}(\vec u_t) - h_{0:t-1} (\vec u_t)\right)\right) 
    \\
    &= \sumT h_{0:t}(\vec x_t) - \sumT h_{0:t-1}(\vec x_t) - \left( \sumT h_{0:t}(\vec u_t) - \sumT h_{0:t-1}(\vec u_t)\right)
    \\
    &= \sumT h_{0:t}(\vec x_t) - \sum_{t=0}^{T-1} h_{0:t}(\vec x_{t+1}) - \left( h_{0:T}(\vec u_T)+ \sum_{t=1}^{T-1} h_{0:t}(\vec u_t) - \sum_{t=0}^{T-1} h_{0:t}(\vec u_{t+1})\right)
    \\
    &\leq \sumT h_{0:t}(\vec x_t) - \sum_{t=1}^{T-1} h_{0:t}(\vec x_{t+1}) - \left( h_{0:T}(\vec x_{T+1})+\sum_{t=1}^{T-1} h_{0:t}(\vec u_t) - \sum_{t=1}^{T-1} h_{0:t}(\vec u_{t+1})-r(\vec u_1)\right)
    \\
    &= \sumT h_{0:t}(\vec x_t) - \sum_{t=1}^{T} h_{0:t}(\vec x_{t+1}) - \left(\sum_{t=1}^{T-1} h_{0:t}(\vec u_t) - \sum_{t=1}^{T-1} h_{0:t}(\vec u_{t+1})\right)+r(\vec u_1).
\end{align}
The inequality holds because of the update rule for each $\vec x_{t+1}$ for any $t$. I.e., 
\[h_{0:T}(\vec x_{T+1})\leq h_{0:T}(\vec u_{T}).
\] Also, $h_0(\vec {u}_{1}) = r(\vec u_1)$.

Writing $h_t$ of the LHS explicitly, and adding the total switching cost to both sides, we get
\begin{align}
    \sumT &\dtp{\vec p_t}{\vec x_t - \vec u_t} +\sumtsw
    \leq \sumT \left( h_{0:t}(\vec x_t) - h_{0:t}(\vec x_{t+1}) \right)+ \sum_{t=1}^{T-1} \left( h_{0:t}(\vec u_{t+1}) -  h_{0:t}(\vec u_t) \right) 
    \\
    &\quad +\sumtsw + r(\vec u_1).
\end{align}
Lastly, by convexity we have that 
\begin{align}
    f_t(\vec x_t) - f_t(\vec u_t) =  \bar f_t(\vec x_t) - \bar f_t(\vec u_t) \leq  \dtp{\vec p_t}{\vec x_t - \vec u_t}
\end{align}
which completes the proof
\end{proof}

We now state the helper Lemma~\ref{lemma:min-of-linearized}, which provides an additional characterization of the \textsc{$k$-LazyGD} iterate using their FTRL from. Specifically, we show that $\vec{\hat {x}}_t$ is the minimizer not only of the original update rule in $(3)$ in the main paper, but also of a related linearized expression.
\begin{lemma}
\label{lemma:min-of-linearized}
For any $\vec x_t \in \X$ satisfying 
\[
\vec x_t \doteq \arg\min_{\vec{x}} h_{0:t-1}(\vec{x}),
\]
it also holds that
\[
\vec x_t \doteq \arg\min_{\vec{x}} \Big(h_{0:t-1}(\vec{x}) + \dtp{\vec g^I_t}{\vec x}\Big),
\]
where $\vec g_t^I$ is chosen according to~\eqref{eq:pruning-cond-n-app}.
\end{lemma}
\begin{proof}
We are given that $\vec x_t = \arg\min_{\vec{x} \in \mathcal{X}} h_{0:t-1}(\vec{x})$. By the first-order optimality condition for constrained minimization (e.g., \citep[Thm.~3.67]{beck-book}), it follows that
\begin{align}
    -\nabla h_{1:t-1}(\vec{x}_t) \in \cone(\vec{x}_t). \label{eq:given}
\end{align}

We now consider the objective $h_{0:t-1}(\vec{x}) + \langle \vec g^I_t, \vec x \rangle$ and verify that $\vec x_t$ also minimizes this function. The corresponding optimality condition is:
\begin{align}
    -\nabla h_{1:t-1}(\vec{y}) - \vec{g}_t^I \in \cone(\vec{y}),
\end{align}
which, when evaluated at $\vec y = \vec x_t$, becomes:
\begin{align}
    -\nabla h_{1:t-1}(\vec{x}_t) - \vec{g}_t^I \in \cone(\vec{x}_t). \label{eq:to-show-op-cond}
\end{align}

We now inspect the two possible values of $\vec g_t^I$ as defined in~\eqref{eq:pruning-cond-n-app}:
\begin{itemize}
    \item[(i)] If $\vec g_t^I = 0$, then \eqref{eq:to-show-op-cond} reduces to \eqref{eq:given}, which holds by assumption.
    \item[(ii)] If $\vec g_t^I = -(\vec p_{1:t-1} + \sigma \vec x_t)$, then
    \begin{align}
        -\nabla h_{1:t-1}(\vec x_t) - \vec g_t^I 
        &= -\vec p_{1:t} - \sigma \vec x_t  - \vec g_t^I  = 0.
    \end{align}
    and hence \eqref{eq:to-show-op-cond} becomes $0 \in \cone(\vec x_t)$, which always holds.
\end{itemize}
In both cases, $\vec x_t$ satisfies the optimality condition for minimizing $h_{0:t-1}(\vec{x}) + \langle \vec g^I_t, \vec x \rangle$ over $\mathcal{X}$, and is therefore a valid minimizer.
\end{proof}

\subsection{Proof of Lemma $7$}
\textbf{Lemma $\mathbf{7}$}
Under the Lipschitzness of the loss functions and compactness of the domain, \ref{k-lazy} iterates guarantee the following bounds on each part of the dynamic FTRL lemma for any $t$:
\begin{itemize}
    \renewcommand\labelitemi{}
  \item $\mathbf{(I)} \le \min\!\big(G\delta_t,\, \tfrac{G^2}{2\sigma}\big)$,
  \item $\mathbf{(II)} \le \tfrac{G}{\sigma}$,
  \item $\mathbf{(III)} \le (2R\sigma + kG)\,\dt$,
\end{itemize}
\vspace{2mm}
\textbf{Lemma $\mathbf{7}$.a}
\[
\mathbf{(I)}  \;\leq\; \min\!\left( G\delta_t,\, \frac{G^2}{2\sigma}\right).
\]
\begin{proof}
By  strong convexity of $h_{0:t}$, for any $\vec s_t \in \partial h_{0:t}(\vec x_t)$,
\begin{align}
h_{0:t}(\vec x_t) - h_{0:t}(\vec x_{t+1})
\;\leq\;
\langle \vec s_t, \vec x_{t+1}-\vec x_t\rangle - \frac{\sigma}{2}\delta_t^2.
\label{eq:sc-actions}
\end{align}
From Lemma~\ref{lemma:min-of-linearized}, the optimality condition of the update step at $\vec x_t$ gives
$0 \in \partial h_{0:t-1}(\vec x_t) + \vec g_t^I$
That is, $-\vec g_t^I \in \partial h_{0:t-1}(\vec x_t)$. Since $\vec p_t = \vec g_t^I + \vec g_t \in \partial h_t(\vec x_t)$, it follows that
\[
\vec g_t \!=\! -\vec g_t^I + (\vec g_t^I+\vec g_t) \in \partial h_{0:t-1}(\vec x_t) + \partial h_t(\vec x_t) \subseteq \partial h_{0:t}(\vec x_t),
\]
where the last step uses the subdifferential sum rule.
Thus we may take $\vec s_t = \vec g_t$ in \eqref{eq:sc-actions}, giving
\begin{align}
\mathbf{(I)} \;\leq\; \langle \vec g_t, \vec x_{t+1}-\vec x_t\rangle \leq G\delta_t. \label{eq:static1}
\end{align}
For the alternative bound, we use 
$ a x - \frac{b}{2}x^2 \;\leq\; \frac{a^2}{2b}, \quad a,b>0$ to obtain
\begin{align}
\mathbf{(I)} \;\leq\; \|\vec g_t\|\delta_t - \frac{\sigma}{2}\delta_t^2
\;\leq\; \frac{G^2}{2\sigma}.
\label{eq:static2}
\end{align}
Combining \eqref{eq:static1} and \eqref{eq:static2}, we get the result 
\end{proof}

\textbf{Lemma $\mathbf{7.b}$}
\[
\mathbf{(II)} \;\leq\; \frac{G}{\sigma}.
\]
To uniformly bound the switching cost and prove this lemma, we recall convex conjugates and some of their key properties.  
With pruning in play, non-regular behavior arises at block boundaries when $n_t=k-1$, (when $t$ is a multiple of $k$). By leveraging the smoothness of the conjugate function, we can handle all time steps uniformly, including those boundary cases.
\paragraph{Definition.}  
The convex conjugate of a function $r$ is defined as
\begin{equation}
    r^\star(\vec g) = \sup_{\vec x \in \mathbb{R}^d} \left\{ \langle \vec g, \vec x \rangle - r(\vec x) \right\}.
\end{equation}

\paragraph{Properties.}  
Let $f$ be a $1$- strongly convex function w.r.t some norm $ \| \cdot \| $
From the classical results of \citet[Lemma 15]{ShalevShwartz2007}, the following hold:
\begin{itemize}
    \item $f^\star (\cdot)$ is differentiable and $1$-smooth with respect to the dual norm $\|\cdot\|_\star$.
    \item $\mathrm{argmin}_{\vec x} \big\{ \langle \vec v, \vec x \rangle + f(\vec x) \big\} = \nabla f^\star(-\vec v)$.
\end{itemize}

Note that, by the above definition and properties, we can write \ref{k-lazy} in yet another way as
\begin{align}
    \vec x_t = \nabla h_0^\star(-\vec p_{1:t-1}).
\end{align}

\begin{proof} (of \textbf{Lemma $7$.b})
The regularizer $h_0$ is $1$-strongly convex w.r.t. the scaled norm $\|\cdot\|_s \doteq \sqrt{\sigma}\|\cdot\|$, hence its conjugate $h_0^\star$ is $1$-smooth w.r.t. $\|\cdot\|_s$.  
By Lemma~\ref{lemma:min-of-linearized}, we may also write
\[
\vec x_t = \nabla h_0^\star(-\vec p_{1:t-1} - \vec g_t^I),
\]
so that\footnote{The additional characterization of Lemma~\ref{lemma:min-of-linearized} is used only for $\vec x_t$, not for $\vec x_{t+1}$.}
\[
\|\vec x_{t+1} - \vec x_t\|_s
= \bigl\|\nabla h_0^\star(-\vec p_{1:t}) - \nabla h_0^\star(-\vec p_{1:t-1}-\vec g_t^I)\bigr\|_s
\;\leq\; \|\vec g_t\|_{s,\star}
= \frac{\|\vec g_t\|}{\sqrt{\sigma}}
\;\leq\; \frac{G}{\sqrt{\sigma}}.
\]
The first inequality follows from the smoothness of $h_0^\star$, and the second from the Lipschitzness assumption.
Hence, 
\[
\sqrt{\sigma}\ \|\vec x_{t+1} - \vec x_t\| \leq \frac{G}{\sqrt{\sigma}},
\] which implies the bound (by multiplying both sides by the positive $1/\sqrt{\sigma}$).
\end{proof}

\textbf{Lemma $\mathbf{7}$.c}
\[
\mathbf{(III)} \;\leq\; (2R\sigma+kG)\dt
\]
\begin{proof}
    From strong convexity
\begin{align}
    \mathbf{(III)} \leq \|\vec q_t\| \|\vec u_{t+1} - \vec u_t\| 
    - \frac{\sigma_{1:t}}{2} \|\vec u_{t+1} - \vec u_t\|^2, \label{eq:h-convex-bound}
\end{align}
where $\vec q_t \in \partial h_{0:t}(\vec u_{t+1})$.

To investigate $\vec q_t$, we expand the subdifferential:
\begin{align}
    \partial h_{0:t}(\vec x) = \vec p_{1:t} + \sigma \vec x + \cone(\vec x).
\end{align}
Choosing the zero vector from $\cone(\vec u_{t+1})$, we obtain:
$\vec q_t = \vec p_{1:t} + \sigma \vec u_{t+1}.$
Since $\|\vec u_{t+1}\| \leq R$, it follows that
\begin{align}
    \|\vec q_t\| \leq \|\vec p_{1:t}\| + R \sigma.
\end{align}

Substituting into \eqref{eq:h-convex-bound}, this yields\footnote{Dropping the negative quadratic may seem wasteful, motivating the use of $ax - \frac{b}{2}x^2 \leq \frac{a^2}{2b}$. However, this hides $x$ as a complexity term. Fixing $x$, even to its minimizer $a/b$, makes the bound linear regardless.} the bound:
\begin{align}
    \textbf{(III)} \leq \left( \|\vec p_{1:t}\| + R \sigma \right) \|\vec u_{t+1} - \vec u_t\|. \label{eq:q_t}
\end{align}
We now derive a bound on the state norm $\|\vec p_{1:t}\|$.
Recall that in the proof of Theorem $1$, we have shown that:
\begin{align}
    - \frac{\vec p_{1:t}}{\sigma} \;=\; \vec x_{t-n_t} - \frac{\vec g_{t-n_t:t}}{\sigma},
\end{align}
of which a direct consequence (from triangular inequality), is 
\begin{align}
    \|\vec p_{1:t} \| &\leq \sigma \|\vec x_{t-n_t}\| +\|\vec g_{t-n_t:t}\| 
     \leq R\sigma+\!\!\sum_{s=t-n_t}^t\!\!\|\vec{g}_s\|
     \\
     & \leq R\sigma + (n_t+1)G \leq R\sigma +kG \label{eq:bounded-state}
\end{align}
The results follows form \eqref{eq:q_t} and \eqref{eq:bounded-state}
\end{proof}
\subsection{Proof of Theorem $5$}
\label{app:proof-thm-5}
\textbf{Theorem $\mathbf{5}$}. For any sequence of convex losses and comparators, the dynamic regret with switching cost of \textsc{$k$-lazyGD} satisfies
\begin{align}
    &\R_T \leq \sumT \min\big( G\delta_t,\, \frac{G^2}{2\sigma}\big) + (2R\sigma+kG)P_T 
+ \frac{\sigma R^2}{2} + \sumT \min(\delta_t, \frac{G}{\sigma}), 
\end{align}
where $\delta_t\!\doteq\!\|\vec x_{t+1}\!-\!\vec x_t\|$.
Moreover, setting  $k = \lfloor c\sqrt{2RT/P_T}\rfloor, c\geq1$ and 
$\sigma = \sigma^\star \doteq \sqrt{(G^2{+}2G)T\,/\,(4RP_T{+}R^2)}$ gives $\R_T = \mathcal{O}\big(\sqrt{(P_T+1)T}\big). 
$
\begin{proof}
Starting from the result of Lemma $6$, and substituting the bounds for each part from Lemma $7$, we obtain 
\begin{align}
    \R_T &\leq \sumT \min\big( G\delta_t,\, \frac{G^2}{2\sigma}\big)\;+\; (2R\sigma+kG)P_T\;+\;\frac{\sigma R^2}{2}\;+\;\sumT \min(\delta_t, \frac{G}{\sigma}), 
    \\
    &\leq \frac{G^2T}{2\sigma}+\frac{\sigma}{2}R^2 + 2R\sigma P_T + k\,GP_T +\frac{G}{\sigma}T \label{eq:sigma-bound}
    \\
    &\;=\; \sqrt{T}\,\sqrt{(G^2+2G)(R^2+4RP_T)} + G\sqrt{T P_T}.
\end{align}
where we used that $r(\vec u_1) \leq \sigma \frac{R^2}{2}$. The second inequality follows by selecting the second branch of the $\min$ terms, and setting $k\leq c\sqrt{T/P_T}$. The last equality follows by substituting the optimal $\sigma$: $\sigma^\star = \sqrt{\frac{T\,(G^2+2G)}{\,R^2+4RP_T\,}}$.
\end{proof}

\textbf{Remark.}
\textit{Note on the choice of $k$.} As noted in Footnote~11, the choice
$ k=\lfloor c\sqrt{T/P_T}\rfloor $
implicitly assumes $P_T>0$. In the static case $P_T=0$, one may simply take $k=T$. More generally, this is also harmless whenever $P_T\le T^{-1/2}$, since then
$ kP_T \le TP_T \le \sqrt{T},$
so the overall regret order in \eqref{eq:lower-bound-max} is unchanged.

\section{THE ENSEMBLE FRAMEWORK ANALYSIS}
\label{app:meta}
As mentioned in the paper, we use the meta learner of SAder from \cite{zhang2021revisiting}, but over an ensemble that vary \emph{both} the regularization/learning rate and the laziness slack (see Fig.\ref{fig:meta} below). For completeness, we provide their algorithm and its guarantee below.
\begin{figure}[h]
    \centering
    \includegraphics[width=0.45\textwidth]{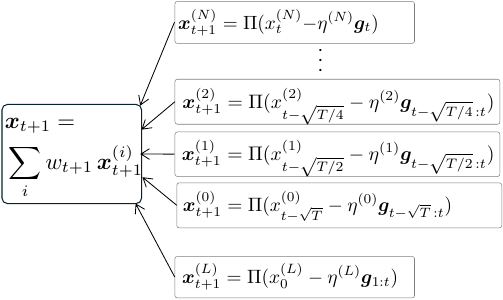}
\caption{Illustration of the meta-learning ensemble at a round $t>\sqrt{T}$. Each base learner corresponds to a different pair of hyperparameters: a learning rate/regularization scale and a laziness slack. Negative indices are interpreted as $0$, and non-integer indices are rounded down so that all update ranges are well defined. The last lazy learner $(L)$ corresponds to the fully lazy endpoint, i.e., the static case $P_T=0$; see the remark in Appendix~\ref{app:proof-thm-5}.}
    \label{fig:meta}
\end{figure}
\begin{algorithm}[h]
	\caption{\small{SAder}}
	\label{alg:LADER}
    \textbf{Input}: A step size $\beta$, the set $\mathcal{H}$ containing the regularization/laziness slack for each expert. \\
    \textbf{Output}: $\{\boldsymbol{x}_t\}_{t=1}^T$.
    \begin{algorithmic}[1] 
    \STATE Activate the set of experts $\{E^{\sigma}|{\sigma}\in \mathcal{H}\}$
    \\
    \STATE Maintain $\sigma^{(1)}\geq \sigma^{(2)} \geq \dots \sigma^{(N)}$, and initialize $w_1^{\sigma^{(i)}}=\frac{C}{i(i+1)}$.
    \FOR{$t=1, 2, \dots, T$}
        \STATE Receive $\vec x_t^{\sigma}$ from $E^{\sigma}, \forall \sigma \in \mathcal{H}$
        \STATE Output the weighted average $\vec x_t = \sum_{\sigma\in\mathcal{H}} w_t^{\sigma} \vec x_t^\sigma$
        \STATE Observe loss $f_t(\cdot)$ and record $\vec g_t\in\partial f_t(\vec x_t)$
        \STATE Updates the weight of each expert by \eqref{eq:hedge-update-per-expert}
        \STATE Send $\vec g_t$ to all experts.
    \ENDFOR
\end{algorithmic}
\end{algorithm}

The SAder algorithm follows the following weight update rule:
\begin{align}
\label{eq:hedge-update-per-expert}
\hspace{-10mm}w_{t+1}^{\sigma} = \frac{1}{M} w_{t}^{\sigma}e^{-\beta\ell_t(\vec x_t^\sigma, \vec x_{t-1}^\sigma)},\quad M\doteq \sum_{\sigma \in \mathcal{H}} w_{t}^{\sigma}e^{-\beta\ell_t(\vec x_t^\sigma,\vec x_{t-1}^\sigma) }, \quad \text{where } \ell_t(\vec x, \vec y) \doteq \dtp{\vec g_t}{\vec x - \vec x_t} + \|\vec x - \vec y\|.
\end{align}

\begin{lemma}[Ensemble regret]\cite[Lemma~3]{zhang2021revisiting}
\label{lemma:zhang-revisiting}
Run Algorithm~\ref{alg:LADER} with
\[
\beta \doteq \frac{2}{(2G+1)D}\sqrt{\frac{2}{5T}},
\]
over the expert set of \textsc{$k$-lazyGD} learners:
\begin{align}
    \mathcal H \doteq \Bigl\{ \tfrac{B}{2^{i-1}} : i=1,\dots,N \Bigr\}, 
    \, N \doteq \Bigl\lceil \tfrac{1}{2}\log_2(8T-7)\Bigr\rceil+1.
\end{align}
indexing each \textsc{$k$-lazyGD} expert directly by \(\sigma \in \mathcal H\) and setting \(k = \max(1,\lfloor\tfrac{2R}{G}\sigma\rfloor)\).
Let \(\vec x_t\) be the meta-learner’s decisions and \(\vec x^{\sigma}_t\) those of expert \(\sigma\).
Then, for any \(\sigma \in \mathcal H\),
\begin{align}
\sum_{t=1}^T \Big(\langle \vec g_t, \vec x_t \rangle + \|\vec x_{t+1}-\vec x_t\|\Big)
\;-\;
\sum_{t=1}^T \Big(\langle \vec g_t, \vec x^{\sigma}_t \rangle + \|\vec x^{\sigma}_{t+1}-\vec x^{\sigma}_t\|\Big)
\;\leq\;
(2G+1)\,2R \,\Big(\ln\big(1/w_1^{\sigma}\big)+1\Big)\sqrt{\frac{5}{8}T}.
\end{align}
\end{lemma}
\vspace{6mm}

\subsection{Proof of Theorem $8$}
\textbf{Theorem $\mathbf{8}$}.
For any comparator sequence, running the meta-learner of \citep[``SAder'']{zhang2021revisiting} over a set of \textsc{$k$-lazyGD} experts from $\mathcal H$, guarantees
\[
    \R_T = \mathcal{O}\!\left(\sqrt{(P_T+1)T}\right).
\]

\begin{proof}{(of \textbf{Theorem} $\mathbf{8}$)}
    Recall the regret bound for any \textsc{$k$-lazyGD} expert (i.e., for any $\sigma$) from \eqref{eq:sigma-bound}:
    \begin{align}
        \label{eq:any-sigma-bound-expert}
        \sum_{t=1}^T \Big(\langle \vec g_t, \vec x^{\sigma}_t \rangle + \|\vec x^{\sigma}_{t+1}-\vec x^{\sigma}_t\|\;-\; f_t(\vec u_t) \Big) \leq \frac{G^2T}{2\sigma}+\frac{\sigma}{2}R^2 + 2R\sigma P_T + k GP_T +\frac{G}{\sigma}T.
    \end{align}
In the main paper, we defined the expert $\sigma^{(s)}$ where 
\begin{align}
    \sigma^{(s)}=\frac{B}{2^{s-1}}, \quad s\doteq \frac{1}{2}\lfloor\log\big( 1 + 4P_T/R\big)\rfloor +1
\end{align}
writing the above bound in \eqref{eq:any-sigma-bound-expert} for expert $\sigma^{(s)}$, and recalling $k = \max(1,\lfloor\tfrac{2R}{G}\sigma\rfloor)$
    \begin{align}
        \sum_{t=1}^T &\Big(\langle \vec g_t, \vec x^{\sigma^{(s)}}_t \rangle + \|\vec x^{\sigma^{(s)}}_{t+1}-\vec x^{\sigma^{(s)}}_t\|\;-\; f_t(\vec u_t) \Big) \leq\frac{(\nicefrac{G^2}{2}+G)T}{\sigma^{(s)}} + \sigma^{(s)} \left(\nicefrac{R^2}{2}+2RP_T+\max(\tfrac{GP_T}{\sigma^{(s)}}, 2RP_T)\right)
        \\ 
        &\leq \frac{(\nicefrac{G^2}{2}+G)T}{\sigma^{\star}} + 2\sigma^{\star} \left(\nicefrac{R^2}{2}+2RP_T+\max(\tfrac{GP_T}{\sigma^{\star}}, 2RP_T)\right)\;\leq\; \frac{(\nicefrac{G^2}{2}+G)T}{\sigma^{\star}} + 2\sigma^{\star} \left(\nicefrac{R^2}{2}+2RP_T+2\sqrt{2}RP_T\right)
        \\
        &= \sqrt{(G^2/2+G)T}\,\Big(\sqrt{2RP_T+R^2/2} +\,\frac{5RP_T+R^2/2}{\sqrt{RP_T+R^2/4}}\Big) \; = \; \mathcal{O}\big(\sqrt{(P_T+1)T}\big)
        \label{eq:cancel-2}
        \end{align}
where the second inequality follows since 
\begin{align}
    \frac{B}{2^{s-1}} \,\geq\, \frac{B}{2^{s^*-1}} \,\geq\, \frac{B}{2^s},
    \;\iff\; 
     \sigma^{(s)}  \,\geq\, \sigma^\star \;\geq\; \frac{\sigma^{(s)}}{2},\qquad s^\star \doteq \tfrac{1}{2}\log\!\big(\sqrt{1+4P_T/R}\big)+1.
\end{align}
The third inequality follows from $\sigma^\star \geq \frac{G}{2\sqrt{2}R}$
Lastly, since expert $\sigma^{(s)}$ is in the set $\mathcal{H}$ by construction, we can write the result of Lemma. \ref{lemma:zhang-revisiting}, for this specific expert:
\begin{align}
\hspace{-10mm}\sum_{t=1}^T \Big(\langle \vec g_t, \vec x_t \rangle + \|\vec x_{t+1}-\vec x_t\|\Big)
-
\sum_{t=1}^T \Big(\langle \vec g_t, \vec x^{(\sigma^s)}_t \rangle + \|\vec x^{(\sigma^s)}_{t+1}-\vec x^{(\sigma^s)}_t\|\Big)
\leq
(2G+1)\,2R \,\Big(\ln\big(1/w_1^{\sigma}\big)+1\Big)\sqrt{\frac{5}{8}T}.
\label{eq:cancel-1}
\end{align}
Summing \eqref{eq:cancel-1} and \eqref{eq:cancel-2} cancels the $\sum_{t=1}^T \Big(\langle \vec g_t, \vec x^{\sigma^{(s)}}_t \rangle + \|\vec x^{\sigma^{(s)}}_{t+1}-\vec x^{\sigma^{(s)}}_t\|\Big)$ term. The proof is then completed using 
\begin{align}
    w_1^{\sigma^s} \geq \frac{1}{(s+1)^2}
\end{align}
and hence
\begin{align}
    \ln(\nicefrac{1}{w_1^{\sigma^s}}) \leq 2\ln(s+1)\leq 2 \ln(\log(\sqrt{1+\nicefrac{6P_T}{R}}+1))
\end{align}
\end{proof}

\section{ADDITIONAL DETAILS ON LAZINESS ADVANTAGES}

\subsection{Variance-based Switching Cost}
\label{app:variance-based-sc}
Beyond the stability properties formalized in Propositions~3 and~4, laziness also yields a variance-based switching bound.  
unlike those results, this bound requires no additional assumptions on the domain or on the magnitude/direction of the accumulated gradients.  
Instead of tying movement to the magnitude of the most recent gradient, it relates switching to its deviation from the running mean. This is a classical result for FTRL based on \citet[Lemma~6]{hazan2010extracting}.

\vspace{2mm}
\begin{proposition}
\label{prop:variance}
For any convex domain and sequence of gradients, fix $0\leq n_t< k-1$ (i.e., within a lazy phase).  Let $\vec G_t \doteq \vec g_{t-n_t:t-1}$.
and define the average state
\begin{align}
    \vec \mu_t = \frac{\vec G_t}{n_t}  
\end{align}
where $\vec G_t$ is the sum of gradients accumulated in the current lazy block.  
Then
\begin{align}
    \|\vec x_{t+1} - \vec x_t\|
    \;\leq\; \frac{5}{4\sigma}\|\vec g_t - \vec \mu_t\| + \frac{4R}{n_t}.
\end{align}
\end{proposition}

\begin{wrapfigure}{r}{0.5\textwidth} 
    \centering
    \includegraphics[width=0.45\textwidth]{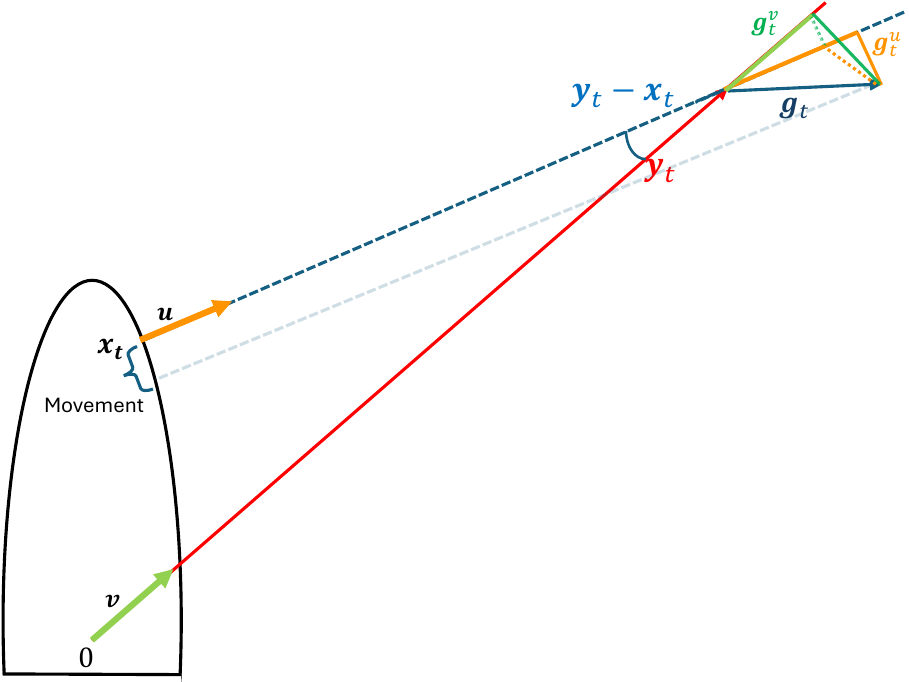}
    \caption{Geometric illustration for Prop.~3.}
    \label{fig:prop3}
\end{wrapfigure}

\begin{proof}
     
We distinguish two cases based on the growth of $\vec G_t$

\textbf{Case 1: $\|\vec G_t\| \leq 4R\sigma$}:
    \begin{align}
        \|\vec x_t - \vec x_{t+1}\| &\stackrel{\text{Lem. 7.b}}{\leq} \frac{1}{\sigma}\|\vec g_t\| = \frac{1}{\sigma} \|\vec g_t- \vec\mu_t+\vec \mu_t\|
        \\
        &\leq\frac{1}{\sigma}\|\vec g_t-\vec \mu_t\| + \frac{1}{\sigma} \|\vec \mu_t\|
        \\
        &\leq \frac{1}{\sigma}\|\vec g_t-\vec \mu_t\| + \frac{4R\|\vec \mu_t\|}{\|\vec G_t\|} = \frac{1}{\sigma}\|\vec g_t-\vec \mu_t\| + \frac{4R}{n_t} 
    \end{align}

\textbf{Case 2: $\|\vec G_t\| \geq 4R\sigma$}:
Here, we show that 
\begin{align}
     \|\vec x_t - \vec x_{t+1}\| \leq \frac{1}{\sigma}\|\vec g_t-\vec \mu_t\| + \frac{R\|\vec g_t\|}{\|\vec G_t\|}
\end{align}
From which the result follows since by the case assumption: 
\begin{align}
    \frac{R\|\vec g_t\|}{\|\vec G_t\|} \leq \frac{R\|\vec g_t-\vec \mu_t\| + R\|\vec \mu_t\|}{\|\vec G_t\|} \leq \frac{1}{4\sigma}\|\vec g_t-\vec \mu_t\|  + \frac{R}{n_t}.
\end{align}
Consider the triangle $0, \vec x_t, \vec y_t$, and note that by the sines law 
\begin{align}
\sin(\theta)=\frac{\|\vec x_t\|\sin(\varphi)}{\|\vec y_t\|} \leq \frac{\sigma R}{\|\vec G_t\|},  \label{eq:sin-bound}  
\end{align}
where $\theta$ is the angle between $\vec y_t$ and $\vec y_t-\vec x_t$, or equivalently, between their unit vectors $\vec v$ and $\vec u$, respectively (see Fig. \ref{fig:prop3}). Let $\vec g_t = \vec g^{\vec{v}}_t + \vec g^{\vec{u}}_t$ be the component of the gradient along the said directions.

Consider the point $\vec y_t - \frac{1}{\sigma}\vec g_t^{\vec u}$, we know that: 

\noindent$(i)$ it is outside $\vec X$: 
\begin{align}
    \|\vec y_{t}-\frac{1}{\sigma}\vec g^{\vec{u}}_t\| \geq \|\vec y_{t}\| - \frac{1}{\sigma} \|\vec g^{\vec{u}}_t\| \geq \frac{1}{\sigma}\|\vec G_t\| - \frac{G}{\sigma} \geq 4R - \frac{G}{\sigma} \stackrel{(\sigma\geq \nicefrac{G}{2\sqrt{2}R})}{\geq} R. 
\end{align}
Where we used a crude lower bound on $\sigma=\sqrt{(G^2+2G)T/(4RP_T+R^2)}$.

$(ii)$ its projection to $\X$ is $\vec x_t$ since  it is on the same line as $\vec y_t - \vec x_t$.
Thus, 
\begin{align}
    \|\vec x_t - \vec x_{t+1}\| \leq \|\vec y_{t+1} - (\vec y_{t}-\frac{1}{\sigma}\vec g^{\vec{u}}_t)\|= \frac{1}{\sigma}\|\vec g_t- \vec g^{\vec{u}}_t\|.
\end{align}
Lastly,  
\begin{align}
\frac{1}{\sigma}\|\vec g_t- \vec g^{\vec{u}}_t\| &\stackrel{(a)}{\leq}\frac{1}{\sigma}\|\vec g_t- \vec z\| \stackrel{(b)}{\leq}\frac{1}{\sigma}\|\vec g_t- \vec g^{\vec v}_t\| + \frac{1}{\sigma}\|\vec g^{\vec v}_t - \vec z\|
\\
&\stackrel{(c)}{\leq}\frac{1}{\sigma}\|\vec g_t- \vec \mu_t\| + \frac{1}{\sigma}\|\vec g^{\vec v}_t - \vec z\| = \frac{1}{\sigma}\|\vec g_t- \vec \mu_t\| + \frac{1}{\sigma}\|\vec g^{\vec v}_t \|\sin(\theta)
\\
&\stackrel{(d)}{\leq} \frac{1}{\sigma}\|\vec g_t- \vec \mu_t\| + \frac{R\|\vec g_t\|}{\|\vec G_t\|},
\end{align}
where $(a)$ holds $\forall \vec z$ in the span $\vec u$, $(b)$ by the triangular inquality, $(c)$ since $\vec \mu_t$ is in the span of $\vec v$, and $\vec g_t^{\vec v}$ is the projection of $\vec v_t$ onto the span of $\vec v$, and $(d)$ from $\|\vec g^{\vec v}_t\|\leq\|\vec g_t\|$ and \eqref{eq:sin-bound}.
\end{proof}

\textbf{Remarks.}  
For the fully lazy case ($k=T$), the result shows that switching can be bounded as  
\[
\Theta\!\left(\,\|\vec g_t - \tfrac{1}{t}\vec g_{1:t}\| + \tfrac{1}{t}\,\right),
\]  
namely in terms of the deviation of the most recent gradient from the running mean, rather than its raw magnitude.  
In the $k$-lazy case, the same principle applies phase by phase, with the bound expressed relative to the average within each lazy block.  
This perspective, originally from \citet[Lemma~6]{hazan2010extracting}, shows an additional laziness advantage that naturally attenuates unnecessary movement by exploiting gradient concentration, without requiring the additional domain or gradient assumptions of Propositions~3 and~4.

\subsection{On the Contractive Property of the Projection Map}
\label{app:contractive-property}
In the main text, we used a simple but useful property of the normalization map 
$\vec{x} \mapsto \frac{\vec{x}}{\|\vec{x}\|}$, which is basically the Euclidean projection onto the unit $\ell_2$ ball.
For completeness, we restate it here with a proof.

Intuitively, the projection of points lying outside the $\ell_2$ ball becomes \emph{closer} as those points move farther away from the ball's boundary. 
This expresses the \emph{contractive} behavior of the projection operator with respect to the distance from the feasible region. 
For clarity, we present the result for the unit ball ($R=1$); the same argument applies to any radius $R > 0$ by simple scaling.
\begin{lemma}
Let $f: \mathbb{R}^n \setminus \{0\} \to \mathbb{R}^n$ be the mapping $f(\vec x) = \frac{\vec x}{\|\vec x\|}$. The Jacobian of $f$, denoted $J_f(\vec x)$ or equivalently $D f(\vec x)$, has an operator norm $\|J_f(\vec x)\|_{\text{op}} = \frac{1}{\|\vec x\|}$.
\end{lemma}

\begin{proof}
Write $f(\vec x) = \vec x\, \|\vec x\|^{-1} = g(\vec x)\, h(\vec x)$ with $g(\vec x)=\vec x$ and $h(\vec x)=\|\vec x\|^{-1}$. 
For any $\vec v \in \mathbb{R}^n$, the product rule gives
\[
J_f(\vec x)\vec v = (J_g(\vec x)\vec v) h(\vec x) + g(\vec x)(J_h(\vec x)\vec v).
\]
Since $J_g(\vec x)=I$, we get $J_g(\vec x)\vec v=\vec v$. The scalar function $h(\vec x)=(\dtp{\vec x}{\vec x})^{-1/2}$ has gradient $\nabla h(\vec x)=-\vec x/\|\vec x\|^3$, so $J_h(\vec x)\vec v = \dtp{\nabla h(\vec x)}{\vec v} = -\tfrac{\dtp{\vec x}{\vec v}}{\|\vec x\|^3}$. Substituting,
\[
J_f(\vec x)\vec v = \tfrac{1}{\|\vec x\|}\vec v - \tfrac{\vec x(\dtp{\vec x}{\vec v})}{\|\vec x\|^3}
= \tfrac{1}{\|\vec x\|}\bigl(\vec v - \vec{\hat{x}}\dtp{\vec{\hat{x}}}{\vec v}\bigr),
\]
where $\vec {\hat{x}}=\vec x/\|\vec x\|$. The map $\vec v \mapsto \vec v - \vec{\hat{x}}\dtp{\vec{\hat{x}}}{\vec v}$ is the orthogonal projection onto $\vec x^\perp$. Thus
\[
\|J_f(\vec x)\vec v\| = \tfrac{1}{\|\vec x\|}\|P_{x^\perp} \vec v\| \le \tfrac{1}{\|\vec x\|}\|\vec v\|.
\]
Taking the supremum over $\|\vec v\|=1$ yields $\|J_f(\vec x)\|_{\mathrm{op}} \le 1/\|\vec x\|$. Hence $\|J_f(\vec x)\|_{\mathrm{op}} = 1/\|\vec x\|$.
\end{proof}

\end{document}